\documentclass[conference]{IEEEtran}
\renewcommand{\IEEEtitletopspaceextra}{0.25in}
\IEEEoverridecommandlockouts
\usepackage{cite}
\usepackage[table]{xcolor}
\usepackage{amsmath,amssymb,amsfonts}
\usepackage{algorithm}
\usepackage{algorithmic}
\usepackage{graphicx}
\usepackage{amsthm}
\newtheorem{thm}{Theorem}

\newtheorem{lemma}[thm]{Lemma}

\newtheorem{assumption}{Assumption}
\usepackage{textcomp}
\usepackage{multirow}
\usepackage[utf8]{inputenc} % allow utf-8 input
\usepackage[T1]{fontenc}    % use 8-bit T1 fonts
\usepackage{hyperref}       % hyperlinks
\usepackage{url}            % simple URL typesetting
\usepackage{booktabs}       % professional-quality tables
\usepackage{nicefrac}       % compact symbols for 1/2, etc.
\usepackage{microtype}      % microtypography
\usepackage{subcaption}
\def\BibTeX{{\rm B\kern-.05em{\sc i\kern-.025em b}\kern-.08em
    T\kern-.1667em\lower.7ex\hbox{E}\kern-.125emX}}
    
\begin{document}

\title{Motion Planning with Model-Based Diffusion via Constraint Optimization and Adaptive Scheduling
}

\author{
\IEEEauthorblockN{
Zhilin He\textsuperscript{1},
Bowei Li\textsuperscript{1},
Jianlin Dou\textsuperscript{1},
Yuner Zhang\textsuperscript{2},
Changliu Liu\textsuperscript{1}
}
\IEEEauthorblockA{
\textsuperscript{1}Carnegie Mellon University,
\textsuperscript{2}University of Pennsylvania
}
\IEEEauthorblockA{
\texttt{\{hectorh, boweili, jdou2, cliu6\}@andrew.cmu.edu},
\texttt{\{yunerzh\}@seas.upenn.edu}
}
\IEEEauthorblockA{Code: \href{https://github.com/hhhhzl/genedynamics}{\texttt{https://github.com/hhhhzl/genedynamics}}}
}

\maketitle

\begin{abstract}
Single-Robot Motion Planning (SRMP) in highly non-convex constrained environments, where robots must satisfy collision-free guarantees, dynamic feasibility, and task-related constraints, is challenging under complex constraints and computational limits. Recent Model-Based Diffusion (MBD) approaches recast the SRMP as trajectory optimization that samples from a posterior over trajectories, using known dynamics, and analytically estimates the score function from rollout samples to guide diffusion denoising toward a low-cost, clean trajectory without demonstration learning. While existing works further adapt MBD to constrained environments and showcase promising performance, they are still limited by (1) enforcing safety either via soft feasibility diffusion priors or hard projection operators, but lack a unified framework to integrate both, and (2) fixing safety enforcement to neglect the changing of diffusion scheduling. Therefore, we introduce Model-Based Diffusion via Constraint Optimization and Adaptive Scheduling (MD-COAS) for SRMP that unifies the inexact Augmented Lagrangian Method (iALM) soft diffusion prior with a Convex Feasible Set (CFS)-based hard projection operator, and adaptively schedules and co-optimizes safety enforcement, along with diffusion scheduling. Experiments demonstrate that our method achieves higher safety \& success rates, faster convergence, and lower final costs than baseline planners on randomly generated highly non-convex 2D benchmarks and a 7-DoF robot arm avoidance task.
\end{abstract}

% \begin{IEEEkeywords}
% Motion Planning, Model-Based Diffusion, Constraint Optimization, Adaptive Scheduling
% \end{IEEEkeywords}

\section{Introduction}
\label{intro}
Single-Robot Motion Planning (SRMP) aims to generate a feasible trajectory from the initial to the goal in complex environments that satisfies obstacle-avoidance, collision-free guarantees, dynamic feasibility, and task-related constraints. Given these requirements, the challenge lies in solving for a collision-free trajectory under complex constraints and computational limits to optimize the task objective. 

Recent approaches recast SRMP as probabilistic inference, sampling from a posterior over trajectories using diffusion models~\cite{ho2020denoising, chi2025diffusion, janner2022planning}, whose score functions are learned from demonstration data, and then steering the diffusion denoising process to recover a clean trajectory. Instead of constructing probabilistic roadmaps (PRM)~\cite{kavraki1996probabilistic, la2011motion} or rapidly exploring random trees (RRT)~\cite{kuffner2000rrt, kleinbort2018probabilistic, karaman2011anytime}, diffusion models exploit demonstration-driven priors to bypass the global geometric search used in sampling-based planners, which typically degrades as environments grow larger and constraints become tighter. However, such diffusion-based methods depend on large-scale demonstration datasets for training and do not rigorously enforce dynamics within the pretrained generative model.

Model-Based Diffusion (MBD)~\cite{pan2024model} addresses this limitation by incorporating explicit dynamics and casting SRMP as a trajectory optimization problem; it can analytically estimate the score from rollout samples and drive the denoising process toward low-cost trajectories. Unlike conventional trajectory optimization methods~\cite{botev2013cross, williams2018information}, which tend to collapse to a single suboptimal solution, MBD retains multimodality through stochastic sampling. Building on MBD, more recent works introduce safety enforcement mechanisms to ensure the generated trajectory satisfies constraints~\cite{mishra2025eb,fishman2023diffusion,cheng2025safe, he2026}. However, they typically either encode constraints as soft feasibility via a diffusion prior or enforce them as a hard post-projection operator, but lack a unified framework that combines both; as a result, they struggle in highly non-convex constrained environments; moreover, they ignore annealing of diffusion scheduling by keeping the safety enforcement fixed from noisy (early) to clean (late) trajectories.

\begin{figure*}[!t]
    \centering
    % Use the pre-rendered PDF so arXiv does not need to invoke Inkscape.
    % The PDF was exported on a letter-size page, so crop it to the artwork.
    \includegraphics[width=\linewidth,trim=18bp 555bp 18bp 25bp,clip]{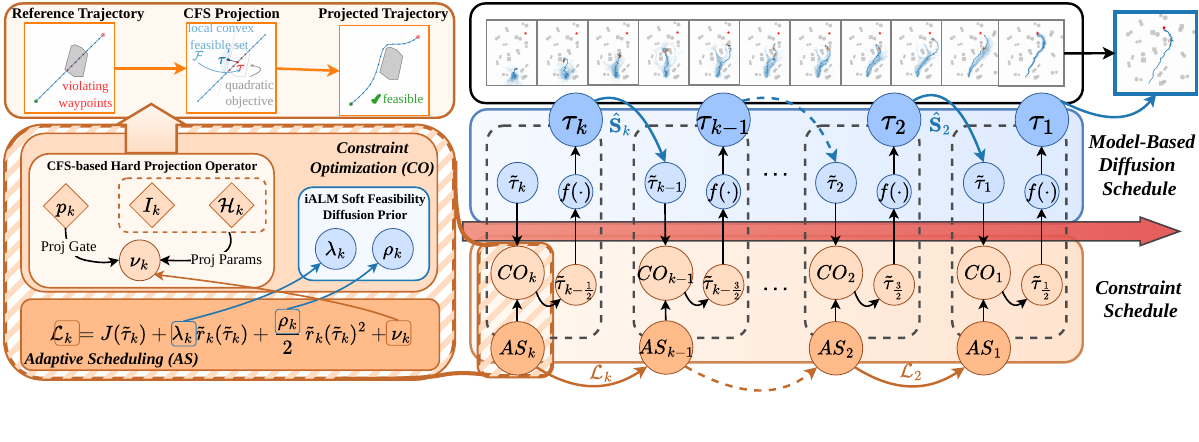}
    \caption{Overview of Constraint Optimization (CO) and Adaptive Scheduling (AS) (Bottom Left) and Model-Based Diffusion for Motion Planning (Bottom Right). Top Left: an example CFS projection correcting a violated reference trajectory. Top Right: an illustration of MD-COAS denoising from noisy samples to a clean trajectory.}
    \label{fig:placeholder}
    \vspace{-0.5cm}
\end{figure*}

We propose a novel motion planning method, \textbf{M}odel-Based \textbf{D}iffusion via \textbf{C}onstraint \textbf{O}ptimization and \textbf{A}daptive \textbf{S}cheduling (MD-COAS), formulating the SRMP problem as optimizing
{%
\setlength{\abovedisplayskip}{2.5pt}
\setlength{\belowdisplayskip}{2.5pt}
\setlength{\abovedisplayshortskip}{1.5pt}
\setlength{\belowdisplayshortskip}{1.5pt}
\begin{equation}
\begin{aligned}
& \tau^\star \in \arg\min_{\tau^{1:T}}\ \  \mathcal{J}(\tau^{1:T}) \\
& \quad \text{s.t.} \quad s^{t+1}=f(s^{t},u^{t}), \quad g^{t}(s^{t},u^{t})\le 0,
\end{aligned}
\label{eq:ocp}
\end{equation}
}%
where $\mathcal{J}(\tau^{1:T})$ is the task objective and trajectory $\tau^{1:T}={(s^t,u^t)}_{t=1}^T$ with states and controls over time horizon $T$. $f(\cdot)$ denotes the discrete-time dynamics, and $g(\cdot)\le 0$ encodes constraints. Following MBD, we interpret reverse diffusion as a solver for~\eqref{eq:ocp}, in which noisy trajectory samples are progressively refined toward low-cost solutions using dynamic models. Our contributions can be summarized as two-fold: 
\textbf{C1:} We develop a constraint-optimization framework for model-based diffusion that integrates an inexact Augmented Lagrangian Method (iALM) soft feasibility prior with a Convex Feasible Set (CFS)-based hard projection operator, thereby preserving informative diffusion guidance and capturing multimodality while certifying safety in highly non-convex constrained environments.
\textbf{C2:} We further propose \textbf{MD-COAS}, in which safety enforcement is adaptively scheduled and co-optimized along with the diffusion scheduling, in a primal--dual view, to account for diffusion annealing and adaptively regulate constraint enforcement under constrained environments.

\section{Related Works}
\vspace{-0.1cm}
\subsection{Diffusion-Based Motion Planning}
 \vspace{-0.1cm}
Existing works on diffusion models for motion planning aim to obtain a clean trajectory by learning the score from demonstration data and then sampling from a posterior over trajectory distributions. Motion Planning Diffusion (MPD)~\cite{carvalho2023motion} formulates trajectory generation as a Denoising Diffusion Probabilistic Models (DDPMs) denoising process~\cite{ho2020denoising, chi2025diffusion,janner2022planning} for recovering a clean trajectory, without explicitly modeling constraints. Building on those methods, more works develop penalty terms, or augmented Lagrangian techniques for soft constraint enforcement~\cite{zhang2025constrained,liang2024multi}, and correct constraint violation by implementing explicit feasibility projection at each denoising step~\cite{romer2024diffusion,christopher2024constrained,xiao2023safediffuser,liang2025simultaneous}. While such methods improve safety, they still rely on large-scale demonstrations and do not rigorously enforce dynamics.

Model-Based Diffusion (MBD)~\cite{pan2024model} refines trajectories through model rollouts to enforce dynamics, driving the denoised trajectories toward low-cost solutions. More recent works ensure safety through an Emerging Barrier (EB), a Control Lyapunov Barrier Function (CLBF), or Control Barrier Function (CBF)-based projections inside the model-based diffusion rollouts~\cite{mishra2025eb,fishman2023diffusion,cheng2025safe, he2026}. However, those methods either lack a unified framework that combines soft feasibility priors with hard post-projection operators or neglect annealing of the diffusion denoising process by keeping safety enforcement fixed throughout the denoising process.
\vspace{-0.2cm}
\subsection{Safe Planning via Constraint Optimization}
\vspace{-0.1cm}
Constraint optimization has long been a principled approach for safe motion planning, often relying on Quadratic Programming (QP)-based trajectory optimization, formulating it as a sequence of convex subproblems via Sequential convexification~\cite{schulman2014motion,bonalli2019gusto} often implemented in a Sequential Quadratic Programming (SQP) framework~\cite{betts2010practical}. Convex Feasible Set (CFS) methods and related QP-based formulations~\cite{liu2018convex} further improve computational efficiency and widely adapt to real-time planning and Model Predictive Control (MPC)~\cite{zhou2021distributed}. Safety filtering frameworks such as Control Barrier Function (CBF)-based QPs~\cite{ames2016control} provide certified safety by enforcing barrier constraints through a QP. While these deterministic approaches offer feasibility guarantees, they are typically local, sensitive to initialization, fixed hyperparameters, and can converge to suboptimal solutions in highly non-convex environments. Primal--dual methods~\cite{wright1997primal} regulate feasibility in constrained optimization via adaptive updates of multipliers and penalty parameters, spanning augmented Lagrangian methods~\cite{sahin2019inexact,howell2019altro,jallet2022constrained} and Alternating
Direction Method of Multipliers (ADMM)-style residual-driven penalty adaptation~\cite{wohlberg2017admm,xu2017admm}. However, none of these methods have been applied to diffusion-based planning to adaptively schedule safety-enforcement parameters throughout the denoising process.

\textbf{Summary.} We are the first to propose a constraint optimization framework that treats the inexact Augmented Lagrangian Method (iALM) as a soft diffusion prior, unifies it with a Convex Feasible Set (CFS)-based hard projection to enforce safety within a stochastic model-based diffusion, and adopts a primal--dual adaptive scheduling strategy that co-optimizes safety enforcement with the diffusion schedule.

\vspace{-0.1cm}
\section{Preliminaries}
\vspace{-0.1cm}

In this section, we cover the relevant background. In what follows, we use a right upper index $t$ for time steps, a right lower index $k$ for reverse diffusion steps, subscript $m$ for the $m$-th Monte Carlo sample, and subscript $j$ for obstacle indices (e.g., $\tau_{k,m}^t$: the trajectory at time $t$ of the $m$-th sample at reverse step $k$, and $\phi_j$: the signed-distance function of obstacle $j$).
\label{sec:preliminaries}

\vspace{-0.1cm}
\subsection{Model-Based Diffusion for Trajectory Sampling}

\label{subsec:mbd}
We treat $\tau^{1:T}$ as a random variable and generate samples by iteratively denoising from Gaussian noise. Let $\{\tau^{1:T}_k\}_{k=1}^{K}$ denote the diffusion variables, where $\tau^{1:T}_1$ is a clean trajectory and $\tau^{1:T}_K$ is approximately pure noise $\mathcal N(0,\mathbf I)$. For the rest of the work, we omit the time index $1:T$ for brevity.

Following DDPMs, the forward process admits the closed-form reparameterization
\begin{math}
\tau_{k+1} = \sqrt{\bar\alpha_k}\,\tau_k + \sqrt{1-\bar\alpha_k}\,\epsilon,
\label{eq:ddpm_reparam}
\end{math}
where $\epsilon\sim\mathcal N(0,\mathbf I)$ with $\bar\alpha_k:=\prod_{j=1}^{k}\alpha_j$ and $\alpha_k\in(0,1]$. 
Let $p_k(\tau_k)$ denote the noisy marginal induced by the forward kernel $q$
\begin{equation}
p_k(\tau_k) := \int q(\tau_k\mid\tau_1)\,p_1(\tau_1)\,d\tau_1.
\label{eq:pk_def}
\end{equation}
A standard result in score-based diffusion is that the optimal reverse-time drift depends on the score
\begin{math}
\mathbf S_k(\tau_k)
=
\nabla_{\tau_k} \log p_k(\tau_k),
\label{eq:score_def}
\end{math}
yielding a reverse update form
{%
\setlength{\abovedisplayskip}{2.5pt}
\setlength{\belowdisplayskip}{2.5pt}
\setlength{\abovedisplayshortskip}{1.5pt}
\setlength{\belowdisplayshortskip}{1.5pt}
\begin{equation}
\tau_{k-1}
=
\frac{1}{\sqrt{\alpha_k}}
\left(
\tau_k + (1-\bar\alpha_k)\,\hat{\mathbf S}_k
\right),
\label{eq:reverse_update_prelim}
\end{equation}
}%
where $\hat{\mathbf S}_k\approx \mathbf S_k$ is an estimator of the score.

In DDPMs, $\hat{\mathbf S}_k$ is computed via a trained noise prediction network. In contrast, Model-Based Diffusion (MBD)~\cite{pan2024model} samples control sequences $u^{1:T-1}$, propagating them through known dynamics to obtain nominal trajectories $\tilde \tau^{1:T}$, and analytically estimating $\hat{\mathbf{S}}_k$ from the resulting denoised rollouts. Concretely, at reverse step $k$, we sample denoised candidates from the Gaussian proposal induced by the forward kernel
\begin{equation}
\tilde{\tau} \sim \Omega_k
:= \mathcal{N}\!\left(
\frac{\tau_k}{\sqrt{\bar{\alpha}_k}},
\left(\frac{1}{\bar{\alpha}_k} - 1\right)\mathbf{I}
\right),
\label{eq:proposal_prelim}
\end{equation}
and form a batch sample $\mathcal{B}_k=\{\tilde{\tau}_{k,m}\}_{m=1}^{M}$. Each candidate is evaluated under the target density $p_1(\cdot)$, yielding the importance-weighted average

\begin{equation}
\bar{\tau}_k
=
\frac{\sum_{m=1}^{M} p_1(\tilde{\tau}_{k,m})\,\tilde{\tau}_{k,m}}
{\sum_{m=1}^{M} p_1(\tilde{\tau}_{k,m})},
\label{eq:mc_avg_prelim}
\end{equation}
which produces the Monte Carlo Score Ascent (MCSA) estimator
\begin{equation}
\hat{\mathbf{S}}_k
\approx
-\frac{\tau_k}{1-\bar{\alpha}_k}
+\frac{\sqrt{\bar{\alpha}_k}}{1-\bar{\alpha}_k}\,
\bar{\tau}_k.
\label{eq:mc_score_prelim}
\end{equation}
Together, \eqref{eq:reverse_update_prelim}--\eqref{eq:mc_score_prelim} define a training-free, guided reverse diffusion process that steers samples toward regions of higher likelihood under $p_1$. The critical step is therefore to construct $p_1(\tau)$ such that its high-density regions coincide with low-cost and dynamically feasible trajectories.

\vspace{-0.2cm}
\subsection{Inexact Augmented Lagrangians as Feasibility Priors}
\label{subsec:pai}
\vspace{-0.1cm}
A common way to soften deterministic optimal control is to introduce a Boltzmann distribution\cite{peters2010relative,levine2018reinforcement} over trajectories:
\begin{math}
p_J(\tau)
\propto
\exp\!\left(-{\mathcal{J}(\tau)}/{\beta}\right),
\label{eq:boltzmann_cost_prelim}
\end{math}
where $\beta>0$ denotes a soft-optimality temperature. Dynamic feasibility and safety can be encoded by a feasibility prior
\begin{math}
p_{\mathcal{F}}(\tau)
\propto\
p_\mathcal{D}(\tau)\,p_\mathcal{G}(\tau),
\label{eq:pfeas_def}
\end{math}
where a hard (indicator) form is
{%
\setlength{\abovedisplayskip}{2.5pt}
\setlength{\belowdisplayskip}{2.5pt}
\setlength{\abovedisplayshortskip}{1.5pt}
\setlength{\belowdisplayshortskip}{1.5pt}
\begin{subequations}
\begin{align}
p_\mathcal{D}(\tau)
&:= \prod_{t=1}^{T-1}\mathbf{1}\!\left(s^{t+1}=f(s^{t},u^{t})\right),\\
p_\mathcal{G}(\tau)
&:= \prod_{t=1}^{T}\mathbf{1}\!\left(g^{t}(s^{t},u^{t})\le 0\right).
\end{align}
\label{eq:target_dist_hard}
\end{subequations}
}%
The resulting target distribution is
\begin{math}
p_1(\tau)
\propto\
p_J(\tau)\,
p_{\mathcal{F}}(\tau),
\end{math}
whose support is restricted to dynamically feasible and collision-free trajectories. However, employing the hard-indicator prior in \eqref{eq:target_dist_hard} leads to a highly sparse weighting landscape, where infeasible trajectories do not contribute signals to the Monte Carlo update. As a result, the sampler lacks meaningful gradient or score information and becomes statistically inefficient, a phenomenon also observed in EB-MBD~\cite{mishra2025eb}.

To retain the inference formulation while avoiding weight collapse, we replace the hard feasibility prior by a soft feasibility kernel $p_\mathcal{S}(\tau)$ that assigns continuous, nonzero weights to mildly infeasible trajectories. Concretely, we parameterize the kernel using an inexact Augmented Lagrangian Method (iALM)~\cite{sahin2019inexact} form: we first aggregate per-step inequality violations via a horizon-normalized quantity
\begin{math}
\bar g(\tau)
=
\frac{1}{T}\sum_{t=1}^{T} [g^{t}(s^{t},u^{t})]_+,
\label{eq:gbar_def}
\end{math}
and define the softened objective
{%
\setlength{\abovedisplayskip}{2.5pt}
\setlength{\belowdisplayskip}{2.5pt}
\setlength{\abovedisplayshortskip}{1.5pt}
\setlength{\belowdisplayshortskip}{1.5pt}
\begin{equation}
\tilde{\mathcal J}(\tau)
=
\mathcal{J}(\tau)
+
\lambda^\top \bar g(\tau)
+
\frac{\rho}{2}\,\|\bar g(\tau)\|^2,
\label{eq:alm_objective}
\end{equation}
}%
where $\lambda\!\ge\!0$ and $\rho\!>\!0$ are dual and penalty parameters. This yields the target distribution
{%
\setlength{\abovedisplayskip}{2.5pt}
\setlength{\belowdisplayskip}{2.5pt}
\setlength{\abovedisplayshortskip}{1.5pt}
\setlength{\belowdisplayshortskip}{1.5pt}
\begin{equation}
p(\tau)
\;\propto\;
\exp\!\left(-{\tilde{\mathcal J}(\tau)}/{\beta}\right)
=
p_J(\tau)\cdot p_{\mathbf{S}}(\tau),
\label{eq:pk_alm}
\end{equation}
and thus an explicit soft feasibility kernel
\begin{equation}
p_{\mathcal{S}}(\tau)
\;\propto\;
\exp\!\left(
-({\lambda^\top \bar g(\tau) + \frac{\rho}{2}\|\bar g(\tau)\|^2 })/{\beta}
\right).
\label{eq:psoft_alm}
\end{equation}
}%
The iALM soft feasibility prior in
\eqref{eq:pk_alm}--\eqref{eq:psoft_alm} mitigates dead samples by assigning
continuous weights to mildly infeasible trajectories. However, \textbf{(i)} in highly
constrained nonconvex environments, purely soft guidance may still stall near
tight passages, and it does not provide strict feasibility; \textbf{(ii)} the static values of $\lambda$ and $\rho$ do not adapt to the diffusion denoise process.

\vspace{-0.2cm}
\section{Methodology}
\vspace{-0.2cm}

We present our Model-Based Diffusion via Constraint Optimization and Adaptive Scheduling (MD-COAS) which makes two contributions: (\textbf{1}) optimize on the soft prior through a Convex Feasible Set (CFS)~\cite{liu2018convex}-based hard projection to maintain safety, (\textbf{2}) develop adaptive constraint scheduling that co-optimizes on both soft prior and hard operator along with diffusion denoising process. 

\vspace{-0.1cm}
\subsection{Problem Statement}
\label{subsec:problem_formulation}
\vspace{-0.1cm}

\newtheorem{definition}{Definition}
\begin{definition}[Single-Robot Motion Planning with Constraints]
\label{def:srmp}
We consider a discrete-time system over the horizon $T$ with state $s^t\in\mathbb{R}^{n_s}$
and control $u^t\in\mathbb{R}^{n_u}$. The trajectory is
$\tau^{1:T}=\{(s^t,u^t)\}_{t=1}^{T}$ and evolves as
\begin{equation}
s^{t+1}=f(s^t,u^t),\quad t=1,\dots,T-1,
\label{eq:dynamics}
\end{equation}
where $f(\cdot)$ is known.
Let $q^t=\Pi_q(s^t)\in\mathbb{R}^d$ denote the workspace position extracted from $s^t$. Given $N_{\mathrm{obs}}$ convex obstacles, we define collision-avoidance constraints using
a signed-distance function $\phi_j:\mathbb{R}^d\to\mathbb{R}$ for each obstacle $j$
\begin{equation}
\phi_j(q^t)\ge 0,\quad \forall t=1,\dots,T,\ \forall j=1,\dots,N_{\mathrm{obs}},
\label{eq:collision_constraints}
\end{equation}
The safe set is
\begin{math}
\Gamma \triangleq \Big\{\{q^t\}_{t=1}^{T}: \phi_j(q^t)\ge 0, \forall t,\forall j\Big\}
\label{eq:safe_set_def}
\end{math}. The objective is to minimize $\mathcal{J}(\tau^{1:T})$ subject to
\eqref{eq:dynamics}--\eqref{eq:collision_constraints}.
\end{definition}

\subsection{Convex Feasible Set as Hard Projection Operator}
\label{subsec:hard_operator}

We introduce a hard feasibility operator $\mathcal{P}_{I,\mathcal{H}}(\cdot)$, parameterized by an inner-iteration budget $I$ and an active-set size $\mathcal{H}$, as a proximal correction step that enforces local safety by projecting nominal diffusion proposals onto a convex inner approximation of the safe set. Given an MBD nominal proposal $\tilde\tau_{k,m}$ at reverse diffusion step $k$ from \eqref{eq:proposal_prelim}, the hard-corrected trajectory is
\begin{math}
\tilde\tau_{k-\frac12,m} = \mathcal{P}_{I_k,\mathcal{H}_k}(\tilde\tau_{k,m}),
\label{eq:hard_operator_apply}
\end{math}
where $\mathcal{P}_{I_k,\mathcal{H}_k}(\cdot)$ is implemented by the Convex Feasible Set (CFS) method~\cite{liu2018convex}, which iteratively convexifies the collision-avoidance constraints and solves a structured QP over the full control sequence.

\begin{definition}[CFS-based Hard Projection $\mathcal{P}_{I_k,\mathcal{H}_k}$]
\label{def:cfs_projection}
Let $\tilde\tau_{k,m}=\{(s_{k,m}^t,u_{k,m}^t)\}_{t=1}^{T}$ be a nominal trajectory proposal and
$q_{k,m}^t$ be the workspace position extracted from $s_{k,m}^t$.
At a CFS inner iteration with the current nominal $u_{k,m}$, we compute a corrected control sequence
$u_{k,m}^\star=\{u_{k,m}^{t}\}_{t=1}^{T-1}$ by solving
\begin{equation}
\begin{aligned}
&  u_{k,m}^\star\in\arg\min_{u}\quad
 \frac{1}{2}\|u-u_{k,m}\|_{R}^2 \\
& \text{s.t.}\quad (a_{j,k,m}^t)^\top u \ge b_{j,k,m}^t,\ \forall (j,t)\in\mathcal{A}_{k,m},
\end{aligned}
\label{eq:cfs_qp_full}
\end{equation}
where we obtain $(a_{j,k,m}^t,b_{j,k,m}^t)$ by first-order rollout linearization around $u_{k,m}$. $R\succ0$ weights proximity to the nominal control sequence, and
$\mathcal{A}_{k,m}$ is an optional active set of obstacle--time pairs with $|\mathcal{A}_{k,m}|\le \mathcal{H}_k$, the top-$\mathcal{H}_k$ most violated pairs ranked by the magnitude of $\phi_j(q_{k,m}^t)$.
\end{definition}

\vspace{-0.2cm}
Around the current nominal point $q_{k,m}^t$, CFS constructs the convex inner approximation as a supporting half-space
\begin{equation}
\mathcal{F}_{j}^t(q_{k,m}^t)
=
\Big\{q:\phi_{j}(q_{k,m}^t)+\hat\nabla \phi_{j}(q_{k,m}^t)^\top (q-q_{k,m}^t)\ge 0
\Big\}.
\label{eq:cfs_halfspace}
\end{equation}
The convex feasible set used by CFS is then
\begin{equation}
\mathcal{F}(q_{k,m}) \triangleq \bigcap_{(j,t)\in\mathcal{A}_{k,m}} \mathcal{F}_j^t(q_{k,m}^t) \ \subset \ \Gamma.
\label{eq:cfs_feasible_set}
\end{equation}

To obtain linear inequalities in the decision variable $u$, we linearize the rollout map around $u_{k,m}$
\begin{equation}
q^t(u)\approx q_{k,m}^t + J_{q,k,m}^t\,(u-u_{k,m}),
\
J_{q,k,m}^t=\left.\frac{\partial q^t}{\partial u}\right|_{u_{k,m}}.
\label{eq:rollout_lin}
\end{equation}
Here $J_{q,k,m}^t$ is the sensitivity of $q^t$ with respect to the stacked control sequence, obtained by differentiating the rollout through dynamics \eqref{eq:dynamics}. Substituting \eqref{eq:rollout_lin} into the CFS constraint in \eqref{eq:cfs_qp_full} yields linear constraints
$(a_{j,k,m}^t)^\top u\ge b_{j,k,m}^t$ with
\begin{equation}
\begin{aligned}
& a_{j,k,m}^t=(J_{q,k,m}^t)^\top \hat\nabla \phi_j(q_{k,m}^t),
\\
& b_{j,k,m}^t=-\phi_j(q_{k,m}^t)+(a_{j,k,m}^t)^\top u_{k,m}.
\end{aligned}
\label{eq:cfs_linear_u}
\end{equation}
Stacking constraints over $(j,t)\in\mathcal{A}_{k,m}$ yields a convex QP with linear constraints $A u \ge b$, producing a minimal-perturbation correction that enforces $\mathcal{F}(q_{k,m})$. Therefore, CFS runs an inner loop for $I_k$ iterations: it forms $\mathcal{F}(q_{k,m})$ by intersecting the half-spaces in \eqref{eq:cfs_halfspace}--\eqref{eq:cfs_feasible_set}, then solves the QP in \eqref{eq:cfs_qp_full} to update $u$ and re-rollout, repeating until convergence or the iteration budget is reached without breaking the dynamics.

\begin{algorithm}[tb!]
\caption{Model-Based Diffusion via Constraint Optimization and Adaptive Scheduling (MD-COAS)}
\label{alg:ours_cfs_mbd}
\small
\begin{algorithmic}[1]
\STATE \textbf{Input:} ${\tau}_K^{1:T}\!\sim\!\mathcal{N}(0,I)$, $s^1$, $s^T$, dynamics $f(\cdot)$, constraints $g(\cdot)\!\le\!0$, diffusion schedule $\{\alpha_k,\bar\alpha_k\}_{k=1}^K$, candidates $M$,
$\lambda_{K}\!\ge\!0$, $\rho_{K}\!>\!0$, $\nu_{K}\!\ge\!0$, tolerance $\xi$, budget $B$
\STATE \textbf{Output:} Clean trajectory ${\tau}_1^{1:T}$

\FOR{$k = K$ \textbf{to} $1$}

    \STATE Sample $M$ nominal candidates $\{\tilde\tau_{k,m}\}_{m=1}^M$ from \eqref{eq:proposal_prelim}.

    \STATE Compute $\tilde r_k$ and $\tilde g(\cdot)$ on $\{\tilde\tau_{k,m}\}_{m=1}^M$.

    \STATE Update soft-prior parameters $\lambda_{k-1},\rho_{k-1}$ \eqref{eq:dual_update_nonmonotone}--\eqref{eq:rho_update_nonmonotone}.

    \STATE Set CFS schedule $(p_k,I_k,\mathcal H_k)$ using $(\tilde r_k,\nu_k)$  \eqref{eq:compute_budget}.
    \STATE Update $\nu_{k-1}$ using $c_k=c(p_k,I_k,\mathcal H_k)$ via \eqref{eq:compute_dual_update}.

    \STATE Sample Bernoulli mask $\delta_{k,m}\sim\textsc{Bernoulli}(p_k)$
    \STATE $\tilde\tau_{k-\frac12,m}\leftarrow 
    \delta_{k,m}\, \mathcal{P}_{I_k,\mathcal H_k}(\tilde\tau_{k,m})
    +(1-\delta_{k,m})\,\tilde\tau_{k,m}$

    \STATE $\tilde\tau_{k-\frac12,m}\leftarrow 
    \textsc{RolloutWithDynamics}(\tilde\tau_{k-\frac12,m}, s^1, f(\cdot))$

    \STATE Compute Monte Carlo mean $\bar{\tau}_{k-\frac12}$ via \eqref{eq:mc_avg_prelim}.
    \STATE Estimate Score $\hat{\mathbf S}_k$ by $\bar{\tau}_{k-\frac12}$ via \eqref{eq:mc_score_prelim}.

    \STATE Monte Carlo Score Ascent update via \eqref{eq:reverse_update_prelim}: \\
    ${\tau}_{k-1} \leftarrow \dfrac{1}{\sqrt{\alpha_k}}\Big({\tau}_{k} + (1-\bar\alpha_k)\,\hat{\mathbf S}_k\Big)$

\ENDFOR
\end{algorithmic}
\end{algorithm}
\subsection{Adaptive Scheduling in Reverse Diffusion Process}
\label{subsec:adaptive_constraint_scheduling}

Fixing constraint enforcement parameters ~\cite{zhang2025constrained,mishra2025eb,romer2024diffusion} ignores that the required correction varies due to diffusion annealing. We therefore introduce adaptive constraint scheduling mechanisms that (\textbf{1}) adaptively optimize the iALM soft prior across diffusion steps and (\textbf{2}) invoke the hard operator only when necessary under a compute budget in a primal--dual view.

At reverse diffusion step $k$, MD-COAS draws $M$ candidate trajectories $\{\tilde\tau_{k,m}\}_{m=1}^{M}$. We summarize feasibility by a robust batch statistic of constraint violation
\begin{equation}
r_k
=
\operatorname{Quantile}_{0.9}\!\Big(\big\{\|\bar g(\tilde\tau_{k,m})\|\big\}_{m=1}^{M}\Big),
\label{eq:primal_residual}
\end{equation}
and introduce a tolerance $\xi>0$ to avoid over-enforcing negligible violations by defining the violation vector
\begin{equation}
\tilde g(\tilde\tau_{k,m}) := [\,\bar g(\tilde\tau_{k,m}) - \xi\mathbf 1\,]_+,
\quad
\tilde r_k := [\,r_k - \xi\,]_+ .
\label{eq:deadzone_def}
\end{equation}
Intuitively, $\tilde r_k$ acts as a step-wise feasibility signal: it is zero when constraints are satisfied up to tolerance, and grows with persistent violations.

We generalize the iALM prior parameters to be diffusion-step dependent and
update them using damped projected ascent and stall-driven penalty adjustment
\begin{equation}
\lambda_{k-1}
=\Pi_{\mathbb{R}_+}\!\left((1-\eta_\lambda)\lambda_k+\rho_k\,\tilde r_k\right),
\ \eta_\lambda\in(0,1),
\label{eq:dual_update_nonmonotone}
\end{equation}
\begin{equation}
\rho_{k-1}=
\begin{cases}
\min(\gamma_\uparrow\rho_k,\rho_{\max}),
& \tilde r_k > \bar r,\\
\max(\rho_k/\gamma_\downarrow,\rho_{\min}),
& \tilde r_k < \underline r, \\
\rho_k, & \text{otherwise},
\end{cases}
\label{eq:rho_update_nonmonotone}
\end{equation}
where $0<\underline r<\bar r$ are residual thresholds for penalty hysteresis. These updates non-monotonically adapt the soft feasibility kernel during reverse diffusion based on the ${\tilde r_k}$ with a scaling factor $\gamma$: feasibility pressure increases when $\tilde r_k$ persists or worsens, and relaxes once $\tilde r_k$ falls below the tolerance, adjusting to the current diffusion step's violation.

\begin{figure*}[tb!]
  \centering
  % --- left label box (fixed width) ---
  \begin{minipage}[c]{0.005\textwidth}
    \centering
    \raisebox{0.5\height}{\rotatebox[origin=c]{90}{\textbf{MBD}}}
  \end{minipage}\hspace{0.005\textwidth}
  % --- row of subfigures ---
  \begin{subfigure}[c]{0.113\textwidth}
    \centering
    \includegraphics[width=\linewidth]{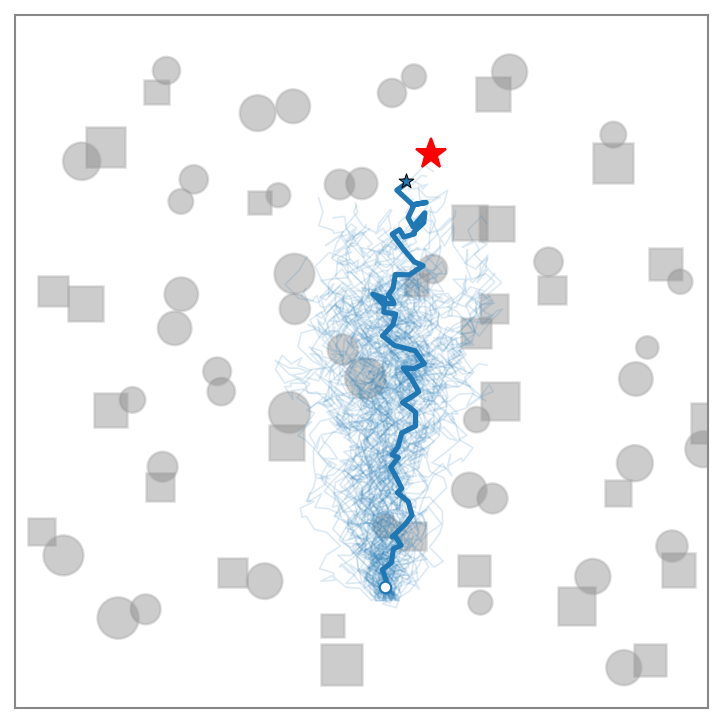}
  \end{subfigure}
  \begin{subfigure}[c]{0.113\textwidth}
    \centering
    \includegraphics[width=\linewidth]{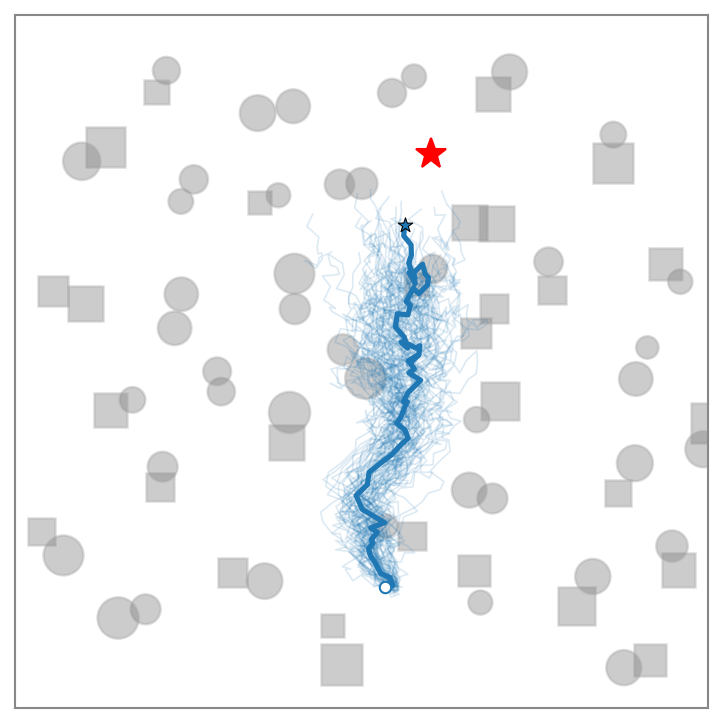}
  \end{subfigure}
  \begin{subfigure}[c]{0.113\textwidth}
    \centering
    \includegraphics[width=\linewidth]{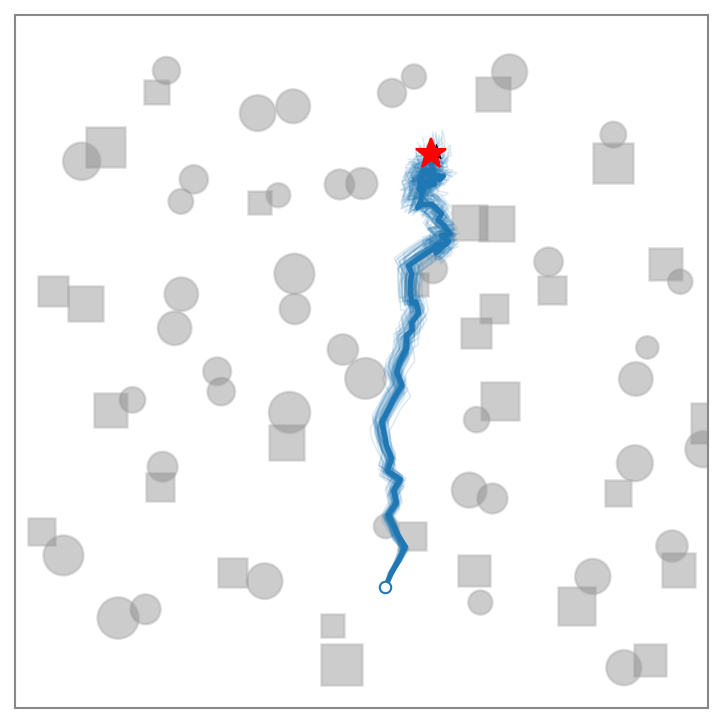}
  \end{subfigure}
  \begin{subfigure}[c]{0.113\textwidth}
    \centering
    \includegraphics[width=\linewidth]{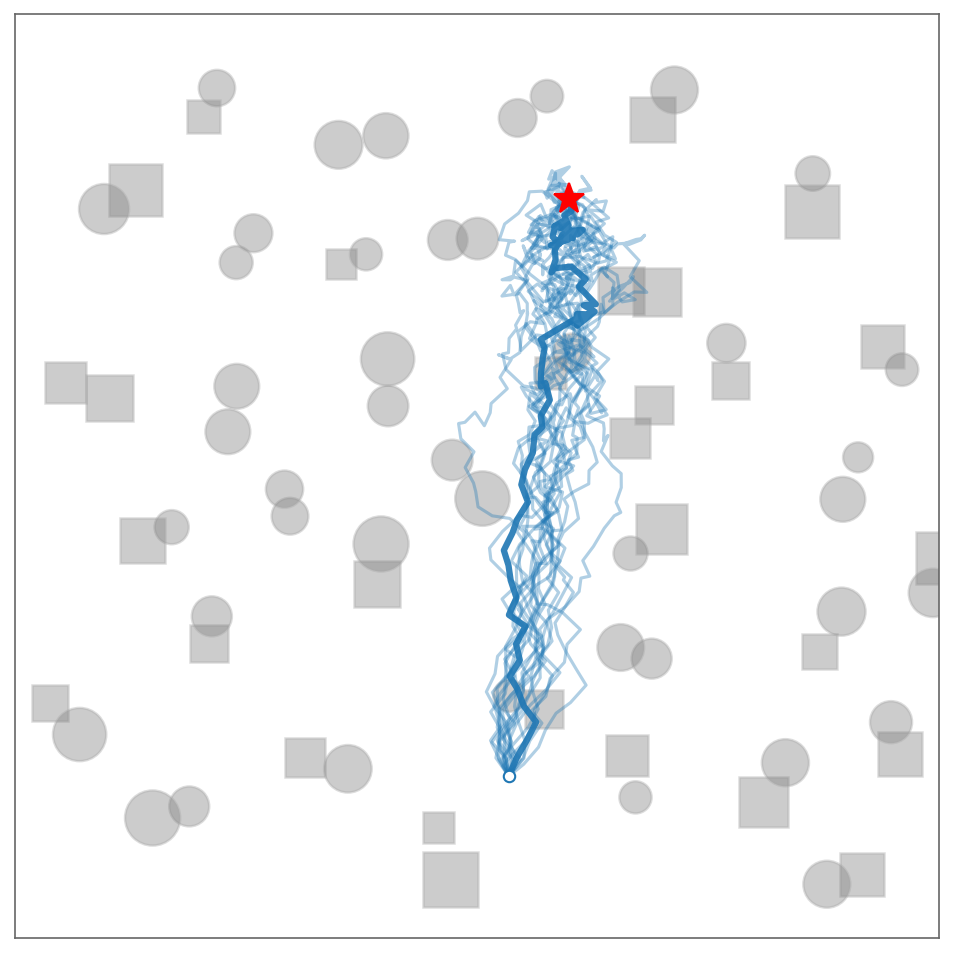}
  \end{subfigure}
  \begin{minipage}[c]{0.005\textwidth}
    \centering
    \raisebox{0.5\height}{\rotatebox[origin=c]{90}{\textbf{MD-COAS-A}}}
  \end{minipage}\hspace{0.005\textwidth}
  % --- row of subfigures ---
  \begin{subfigure}[c]{0.113\textwidth}
    \centering
    \includegraphics[width=\linewidth]{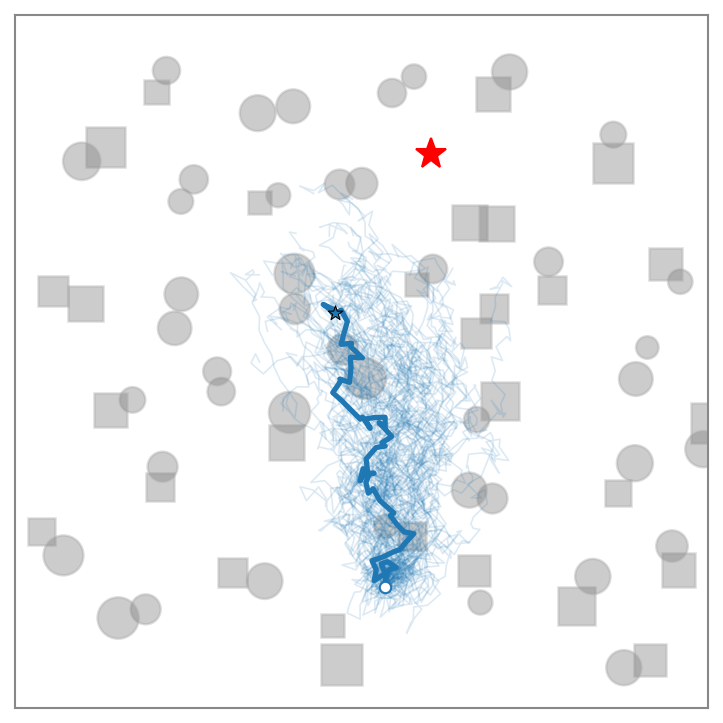}
  \end{subfigure}
  \begin{subfigure}[c]{0.113\textwidth}
    \centering
    \includegraphics[width=\linewidth]{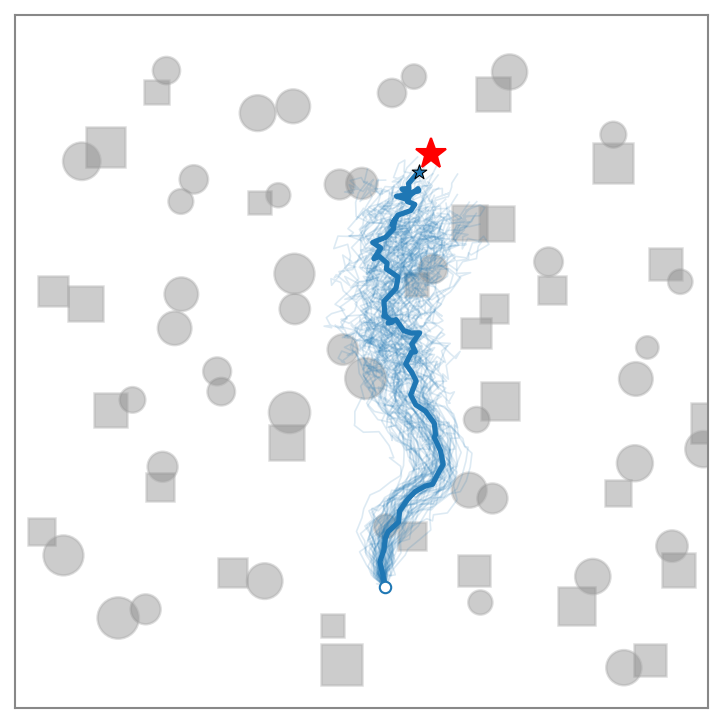}
  \end{subfigure}
  \begin{subfigure}[c]{0.113\textwidth}
    \centering
    \includegraphics[width=\linewidth]{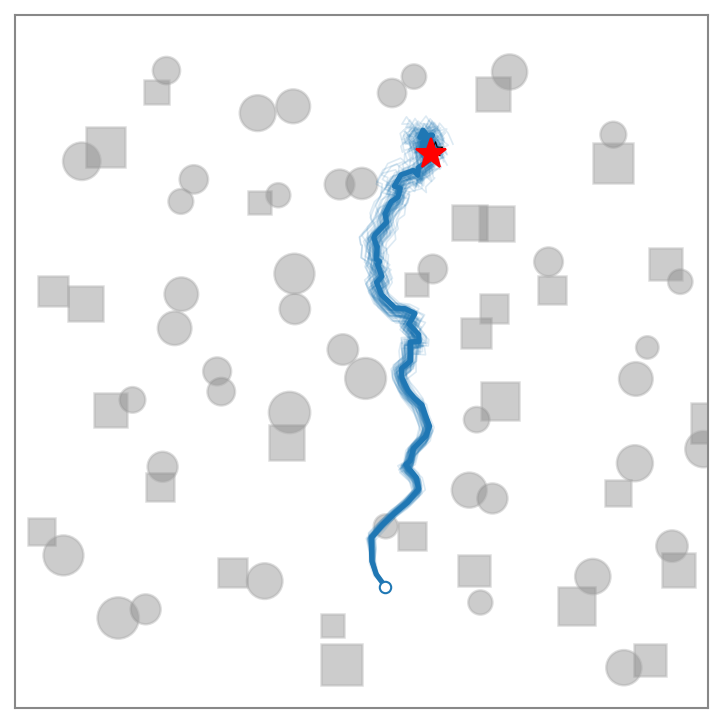}
  \end{subfigure}
  \begin{subfigure}[c]{0.113\textwidth}
    \centering
    \includegraphics[width=\linewidth]{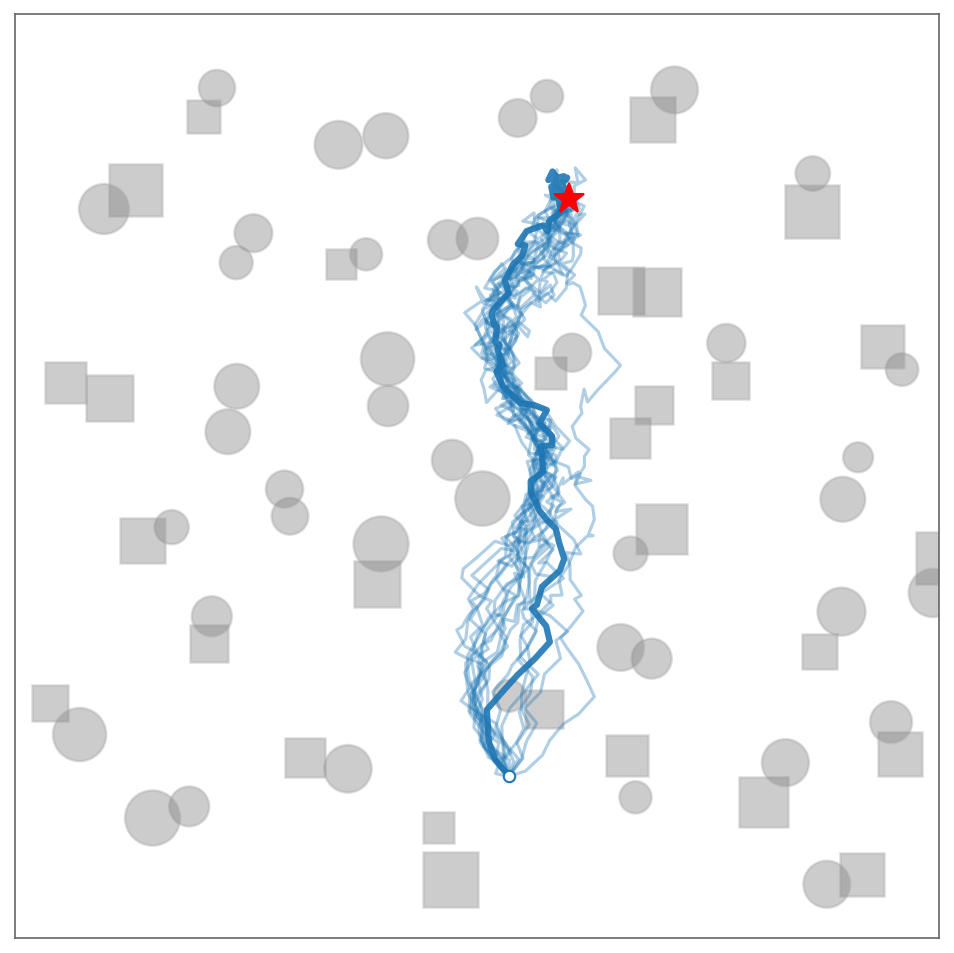}
  \end{subfigure}

  \par\vspace{5pt}

  \begin{minipage}[c]{0.005\textwidth}
    \centering
    \raisebox{0.5\height}{\rotatebox[origin=c]{90}{\textbf{EB-MBD}}}
  \end{minipage}\hspace{0.005\textwidth}
  % --- row of subfigures ---
  \begin{subfigure}[c]{0.113\textwidth}
    \centering
    \includegraphics[width=\linewidth]{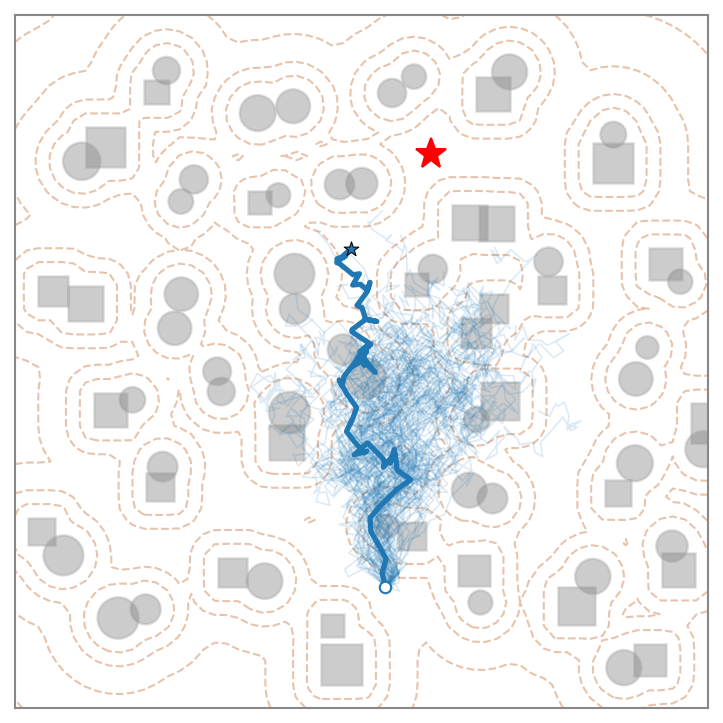}
  \end{subfigure}
  \begin{subfigure}[c]{0.113\textwidth}
    \centering
    \includegraphics[width=\linewidth]{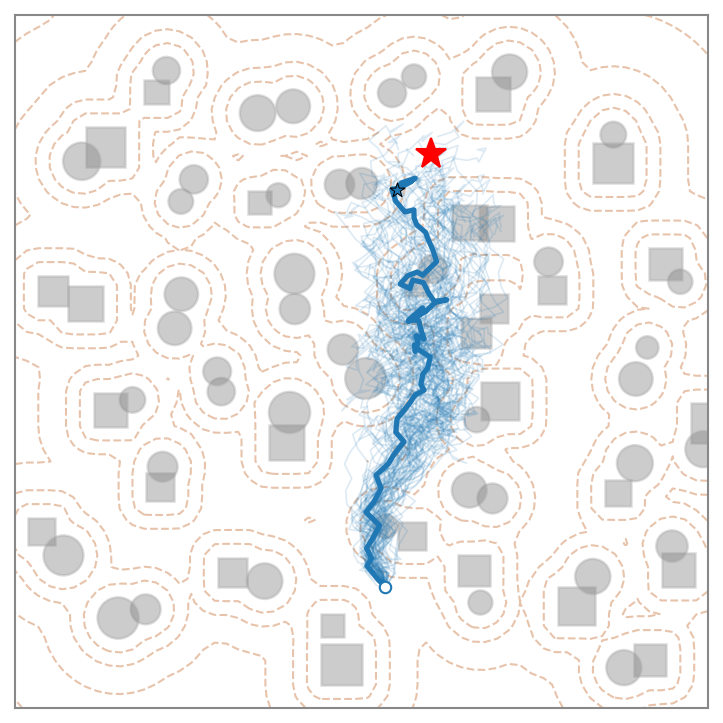}
  \end{subfigure}
  \begin{subfigure}[c]{0.113\textwidth}
    \centering
    \includegraphics[width=\linewidth]{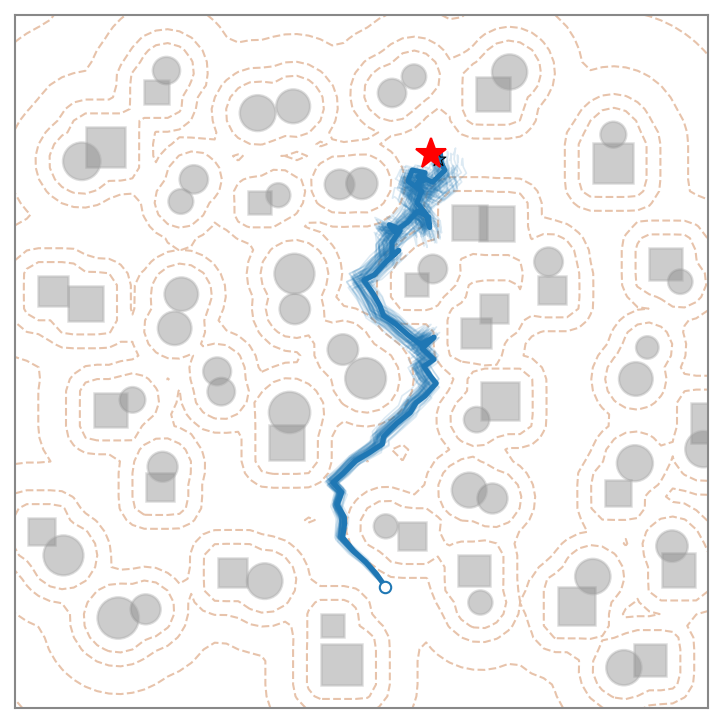}
  \end{subfigure}
  \begin{subfigure}[c]{0.113\textwidth}
    \centering
    \includegraphics[width=\linewidth]{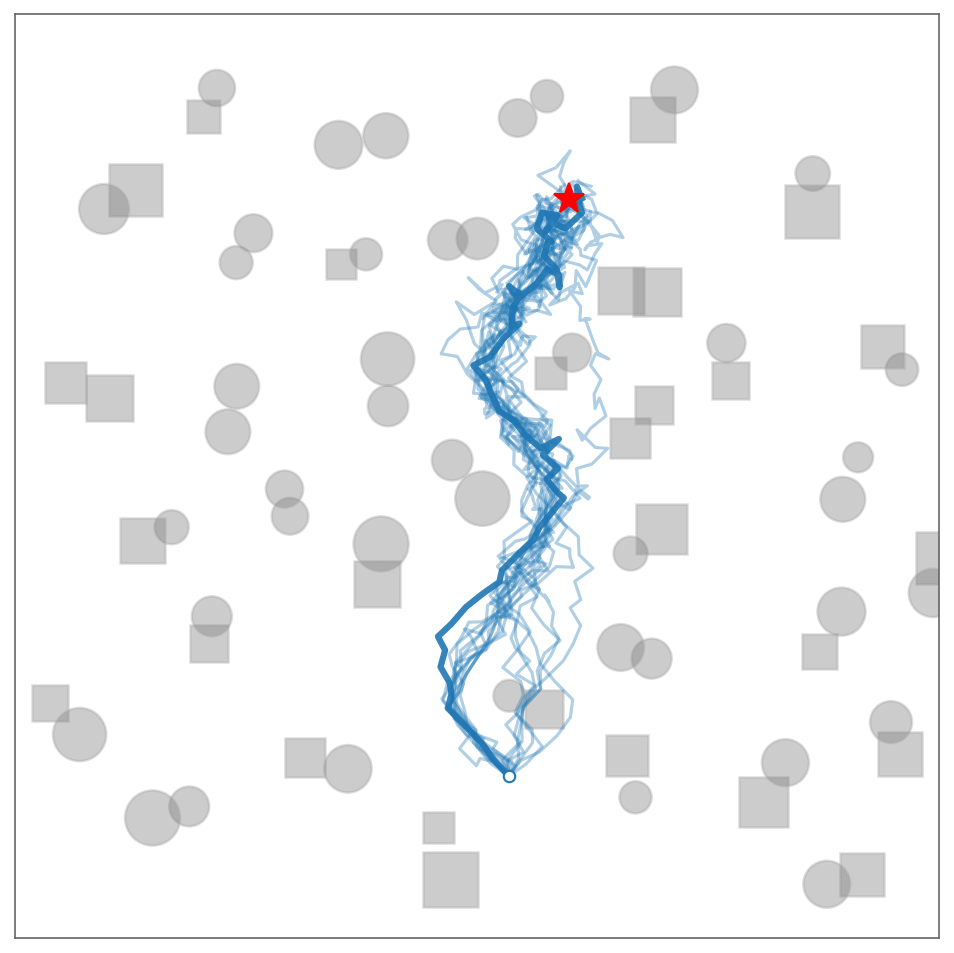}
  \end{subfigure}
  % \par\vspace{5pt}
  \begin{minipage}[c]{0.005\textwidth}
    \centering
    \raisebox{0.5\height}{\rotatebox[origin=c]{90}{\textbf{MD-COAS-F}}}
  \end{minipage}\hspace{0.005\textwidth}
  % --- row of subfigures ---
  \begin{subfigure}[c]{0.113\textwidth}
    \centering
    \includegraphics[width=\linewidth]{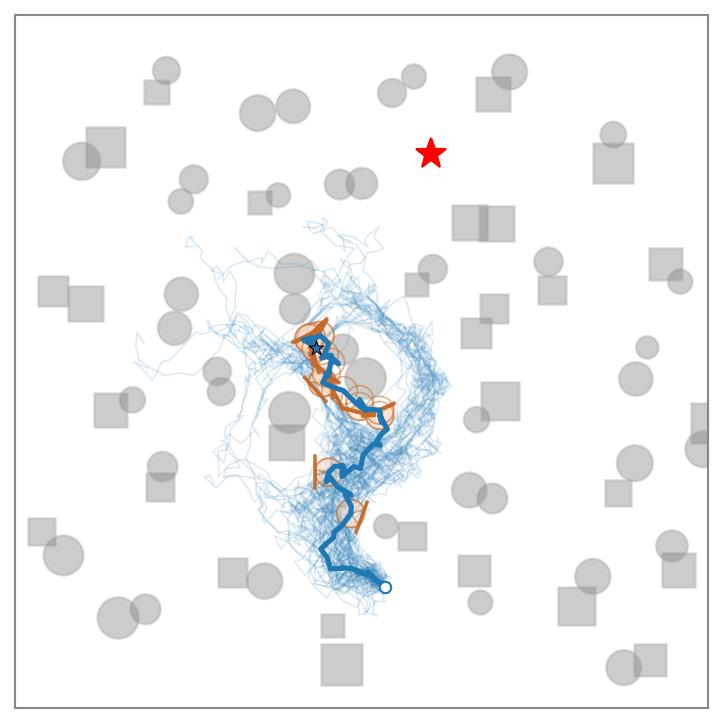}
  \end{subfigure}
  \begin{subfigure}[c]{0.113\textwidth}
    \centering
    \includegraphics[width=\linewidth]{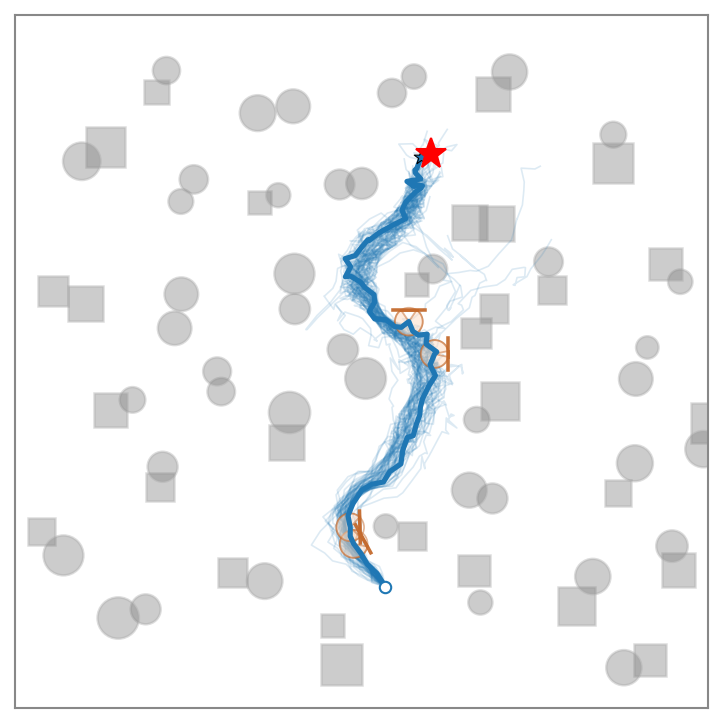}
  \end{subfigure}
  \begin{subfigure}[c]{0.113\textwidth}
    \centering
    \includegraphics[width=\linewidth]{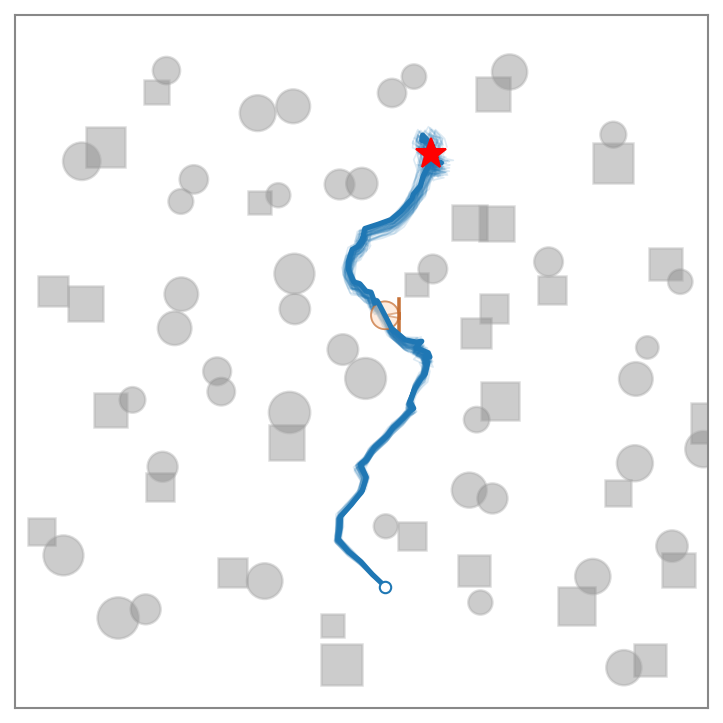}
  \end{subfigure}
  \begin{subfigure}[c]{0.113\textwidth}
    \centering
    \includegraphics[width=\linewidth]{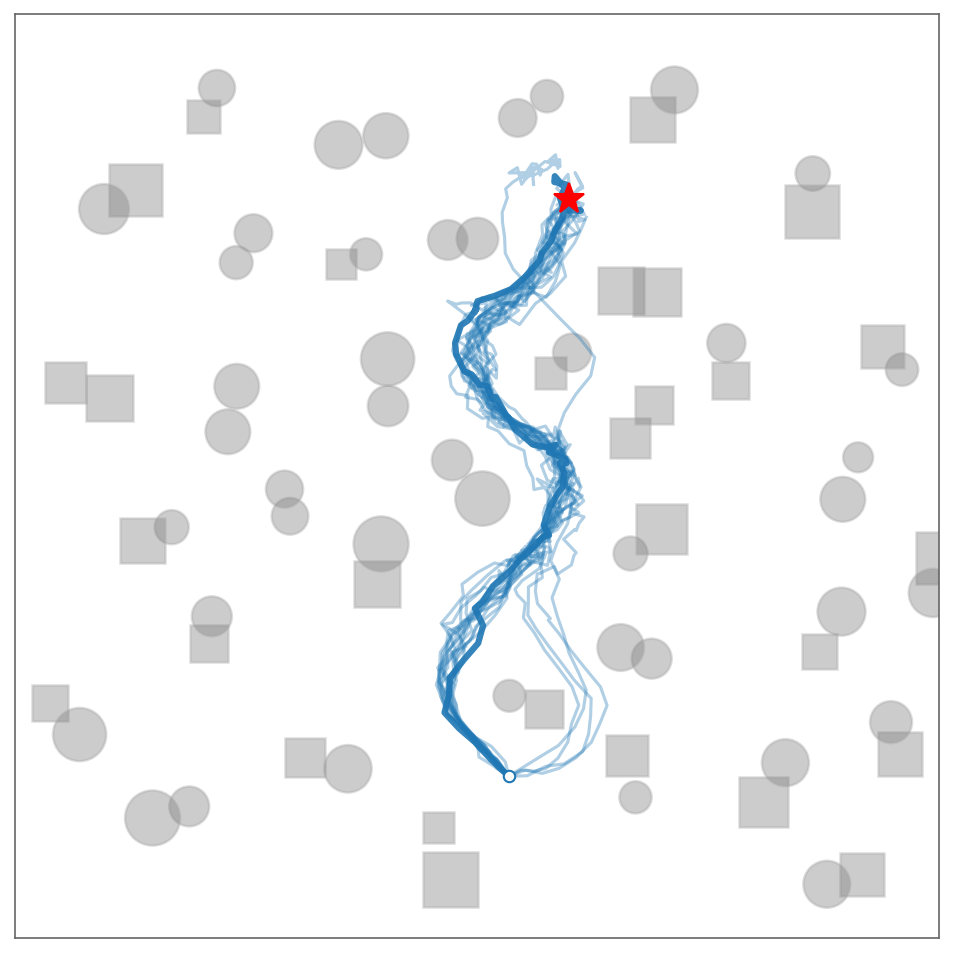}
  \end{subfigure}

  \par\vspace{5pt}

  \begin{minipage}[c]{0.005\textwidth}
    \centering
    \raisebox{0.5\height}{\rotatebox[origin=c]{90}{\textbf{MDOC}}}
  \end{minipage}\hspace{0.005\textwidth}
  % --- row of subfigures ---
  \begin{subfigure}[c]{0.113\textwidth}
    \centering
    \includegraphics[width=\linewidth]{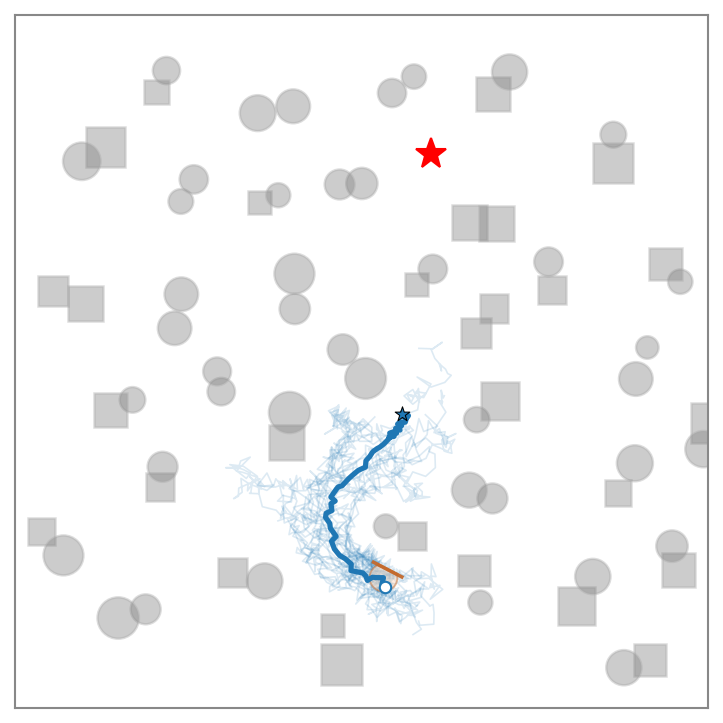}
  \end{subfigure}
  \begin{subfigure}[c]{0.113\textwidth}
    \centering
    \includegraphics[width=\linewidth]{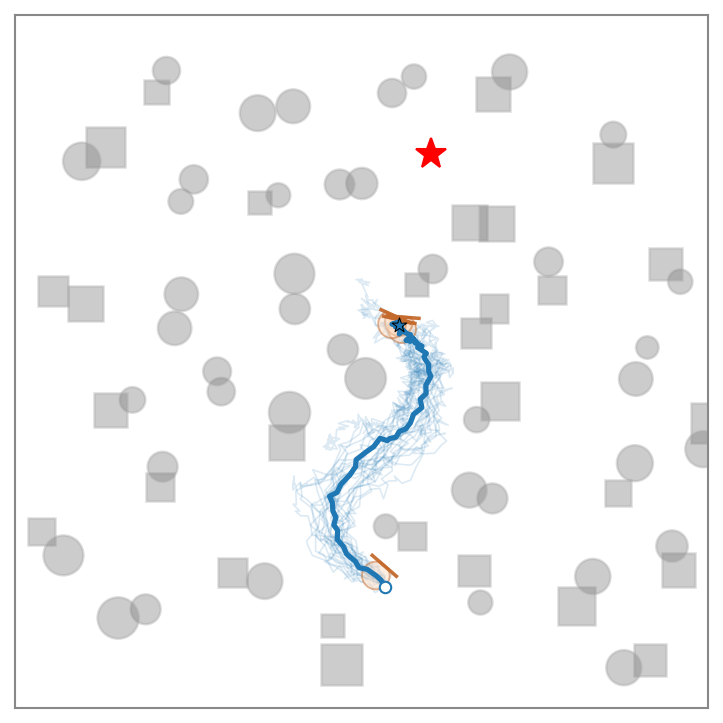}
  \end{subfigure}
  \begin{subfigure}[c]{0.113\textwidth}
    \centering
    \includegraphics[width=\linewidth]{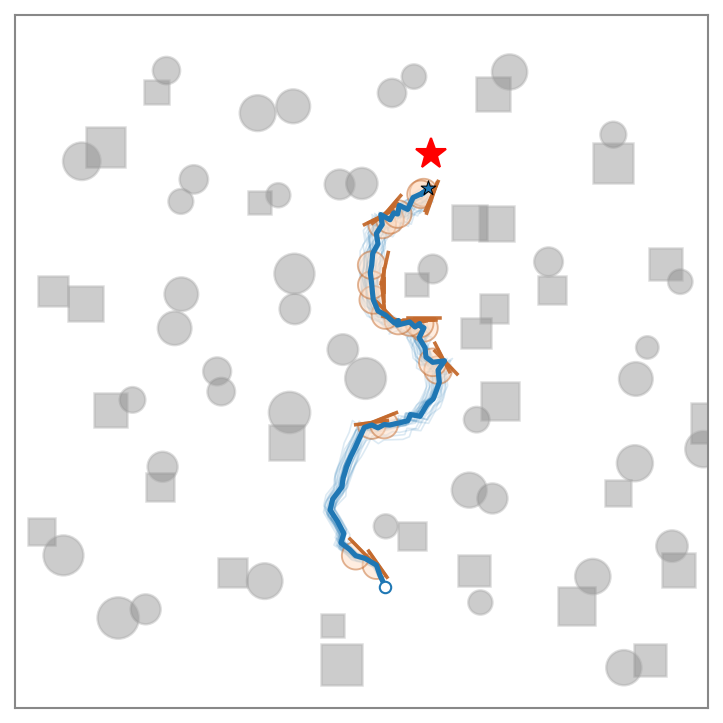}
  \end{subfigure}
  \begin{subfigure}[c]{0.113\textwidth}
    \centering
    \includegraphics[width=\linewidth]{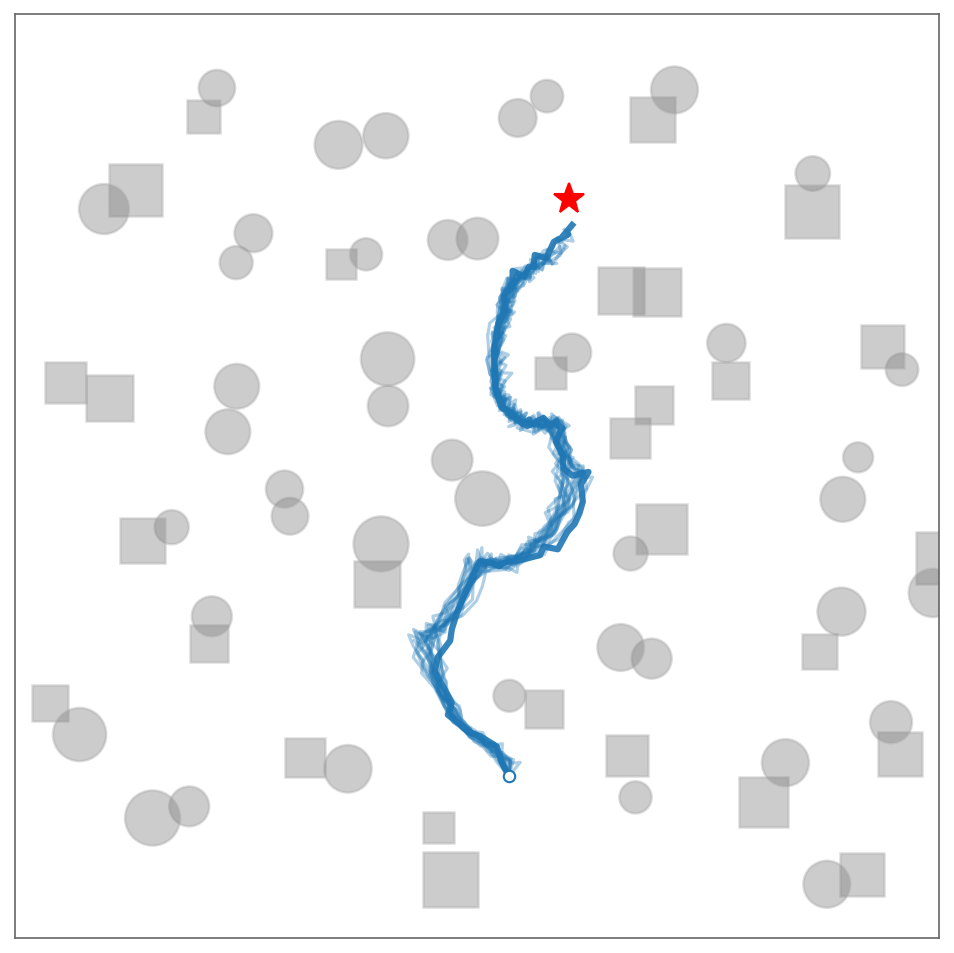}
  \end{subfigure}
  % \par\vspace{5pt}
  \begin{minipage}[c]{0.005\textwidth}
    \centering
    \raisebox{0.5\height}{\rotatebox[origin=c]{90}{\textbf{MD-COAS}}}
  \end{minipage}\hspace{0.005\textwidth}
  % --- row of subfigures ---
  \begin{subfigure}[c]{0.113\textwidth}
    \centering
    \includegraphics[width=\linewidth]{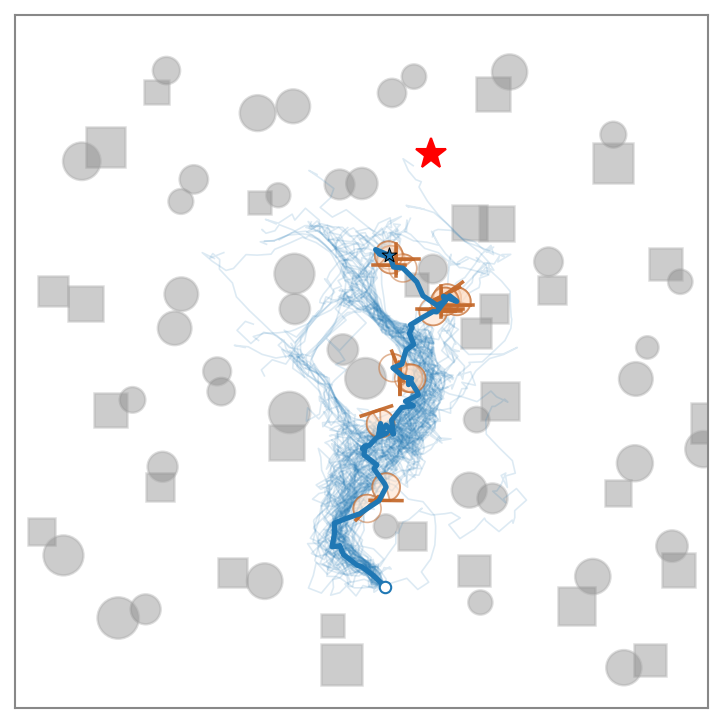}
  \end{subfigure}
  \begin{subfigure}[c]{0.113\textwidth}
    \centering
    \includegraphics[width=\linewidth]{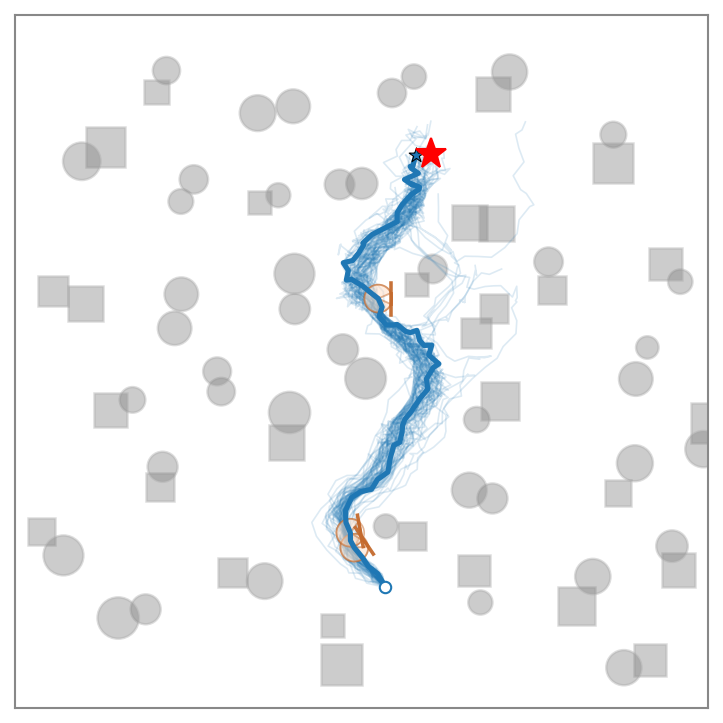}
  \end{subfigure}
  \begin{subfigure}[c]{0.113\textwidth}
    \centering
    \includegraphics[width=\linewidth]{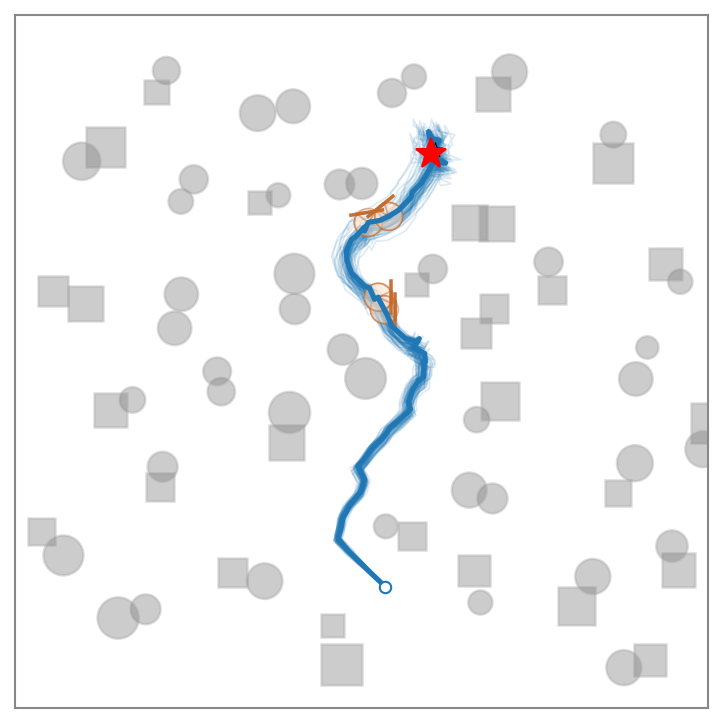}
  \end{subfigure}
  \begin{subfigure}[c]{0.113\textwidth}
    \centering
    \includegraphics[width=\linewidth]{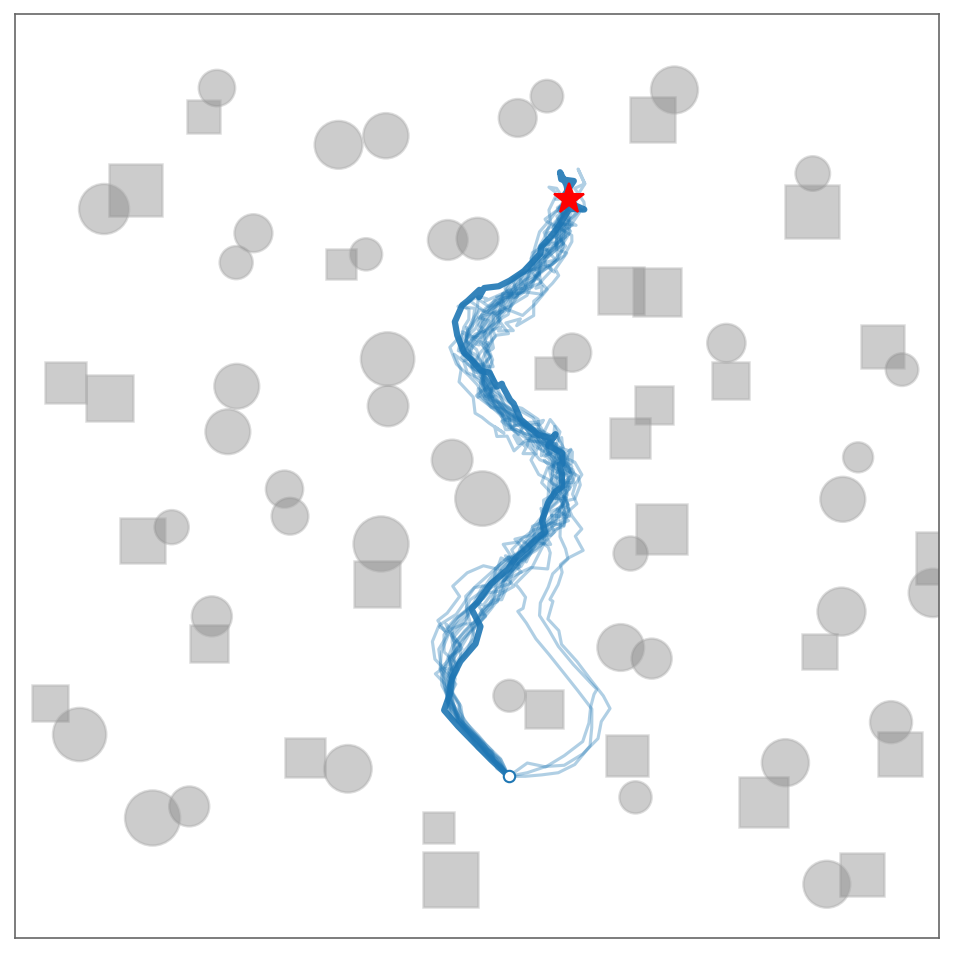}
  \end{subfigure}
  \caption{Illustrations of the diffusion process of model-based methods for 2D Constrained Obstacle environment at L10 (Union Map) and Seed 8. For each method, the left 3 columns denote rollouts at 90\%, 50\%, and 10\% of the denoising process (early → mid → late), respectively, and the rightmost column indicates the final clean trajectories of 20 candidates. We use blue color to represent (sample) trajectories; Orange circles represent static Emerging Barrier in EB-MBD, while hard QP correction margin and boundary in MDOC, MD-COAS-F, and MD-COAS.}
  \label{fig:diffusion_process}
  \vspace{-0.5cm}
\end{figure*}

Likewise, we treat the hard operator as a corrector that is invoked only when needed across diffusion steps. Given a nominal proposal $\tilde\tau_{k,m}$ at step $k$, we apply
\begin{equation}
\tilde\tau_{k-\frac12,m}
=
\begin{cases}
\mathcal P_{I_k, \mathcal{H}_k}(\tilde\tau_{k,m}), & \text{with probability } p_k,\\
\tilde\tau_{k,m}, & \text{otherwise},
\end{cases}
\label{eq:qp_gate}
\end{equation}
where $p_k\in[0,1]$ is the projection probability. We adapt $(p_k,\mathcal H_k,I_k)$ using a batch violation statistic $v_k$: we increase projection effort by raising $p_k$, $\mathcal H_k$, and $I_k$ as $v_k$ increases, and shrinking them as $v_k$ decreases. To bound the average computation, we constrain the expected per-diffusion-step effort by
\begin{math}
\frac{1}{K}\sum_{k=0}^{K-1}\mathbb{E}\!\left[c_k(p_k, I_k, \mathcal{H}_k)\right]\le B.
\label{eq:compute_budget}
\end{math}
Here $c_k$ denotes the expected projection computed at step $k$, which scales with both how often we project ($p_k, I_k$) and the per-QP size controlled by $\mathcal{H}_k$ (e.g., $c_k\approx p_k I_k \mathcal{H}_k$), and $B$ is a fixed budget.

We regulate this budget using a compute dual $\nu_k\ge 0$ with damped projected ascent
{%
\setlength{\abovedisplayskip}{2.5pt}
\setlength{\belowdisplayskip}{2.5pt}
\setlength{\abovedisplayshortskip}{1.5pt}
\setlength{\belowdisplayshortskip}{1.5pt}
\begin{equation}
\nu_{k-1}
=
\Pi_{\mathbb{R}_+}\!\left(\nu_k+\eta_\nu\,(c_k-B)\right),
\quad
\eta_\nu\in(0,1).
\label{eq:compute_dual_update}
\end{equation}
}%
When the realized effort $c_k$ exceeds the fixed budget $B$, $\nu_k$ increases, discouraging aggressive projection; when the effort remains below the budget, $\nu_k$ decreases, allowing more frequent or stronger corrections.

We couple feasibility and compute via an augmented Lagrangian surrogate at diffusion step $k$
\begin{equation}
\begin{aligned}
& \mathcal{L}_k(\tilde\tau_k,\lambda_k,\rho_k,\nu_k, p_k, I_k, \mathcal{H}_k) = \mathcal{J}(\tilde\tau_k) + \\
& \underbrace{\lambda_k\, \tilde r_k(\tilde\tau_k) + \frac{\rho_k}{2}\,\tilde r_k(\tilde\tau_k)^2}_{\text{Adaptive iALM Soft Diffusion Prior}} + \underbrace{\nu_k\big(c_k(p_k, I_k, \mathcal{H}_k)-B\big)}_{\text{Adaptive CFS-based Hard Projection}}
\label{eq:lagrangian_joint}
\end{aligned}
\end{equation}

Therefore, \eqref{eq:lagrangian_joint} integrates the iALM soft diffusion prior and the hard-operator schedule within a unified primal--dual view: the residual sequence $\{\tilde r_k\}$ adjusts $(\lambda_k,\rho_k)$ to modulate feasibility pressure across denoising, while the computed residual $c_k-B$ adjusts $\nu_k$ to regulate projection effort. As a result, MD-COAS co-optimizes safety enforcement and diffusion scheduling, enforcing constraints more aggressively only when violations persist and compute allows.

\subsection{Putting Everything Together: MD-COAS}
\vspace{-0.1cm}
We now combine (\textbf{1}) the MCSA sampler in Section~\ref{subsec:mbd}, (\textbf{2}) the iALM soft diffusion prior in Section~\ref{subsec:pai}, and (\textbf{3}) the CFS-based hard projection operator $\mathcal P_{I_k, \mathcal{H}_k}(\cdot)$ in Section~\ref{subsec:hard_operator}, with (\textbf{4}) adaptive scheduling in Section~\ref{subsec:adaptive_constraint_scheduling} to form our method.

As shown in Algorithm~\ref{alg:ours_cfs_mbd}, at each reverse diffusion step $k$, MD-COAS first draws $M$ nominal candidates from the Gaussian proposal in \eqref{eq:proposal_prelim}, evaluates the batch feasibility residual $\tilde r_k$ via \eqref{eq:primal_residual}--\eqref{eq:deadzone_def}, and updates the soft-prior parameters $(\lambda_{k-1},\rho_{k-1})$ using \eqref{eq:dual_update_nonmonotone}--\eqref{eq:rho_update_nonmonotone}. It then schedules the hard correction effort $(p_k,I_k,\mathcal H_k)$ under the compute dual $\nu_k$ using $c_k=c(p_k,I_k,\mathcal H_k)$ and through \eqref{eq:compute_dual_update}, and selectively applies the CFS projection with the gate in \eqref{eq:qp_gate}; otherwise, it keeps $\tilde\tau_{k-\frac12,m}=\tilde\tau_{k,m}$. Resulting candidates are rolled out through the known dynamics to obtain dynamically feasible trajectories, which are reweighted by the adaptive iALM target in \eqref{eq:pk_alm} to compute the Monte Carlo mean and score $\hat{\mathbf S}_k$ via \eqref{eq:mc_avg_prelim}--\eqref{eq:mc_score_prelim}, and the noisy trajectory is updated by the reverse step in \eqref{eq:reverse_update_prelim}.

\section{Experimental Analysis}
\vspace{-0.1cm}
We explore the performance of our method, MD-COAS, in SRMP by addressing the following questions:

\begin{itemize}
    \item What advantages does MD-COAS have over state-of-the-art (SOTA) model-based diffusion planners? Why are Constraints Optimization (CO) and Adaptive Scheduling (AS) important?
    \item Can MD-COAS maintain performance when scaling to a higher-dimensional 7-DoF robot arm avoidance task?
\end{itemize}

\textbf{Baselines.}
We compare our method, MD-COAS, with representative model-based and model-free diffusion motion planning baselines:
(\textbf{1}) MBD~\cite{pan2024model}, a vanilla model-based diffusion planner for trajectory optimization;
(\textbf{2}) EB-MBD~\cite{mishra2025eb}, a safety-enhanced variant of MBD that mitigates dead samples in highly constrained environments using an emerging barrier;
(\textbf{3}) MDOC~\cite{he2026}, which integrates Control Barrier Function (CBF)-based projection into each MBD rollout step;
(\textbf{4}) SafeDiffuser~\cite{xiao2023safediffuser}, a model-free diffusion planner that embeds CBF-based invariance into the denoising process by solving a per-step QP; (\textbf{5}) DPCC-C~\cite{romer2024diffusion}, a model-free diffusion predictive control baseline that enforces safety via constraint tightening and selects the trajectory with minimum cumulative projection cost. We further include \textbf{two ablation variants:}
(\textbf{i}) MD-COAS-A, which applies adaptive iALM-based soft priors without QP projection;
(\textbf{ii}) MD-COAS-F, which uses fixed constraint optimization parameters without adaptive scheduling.
All experiments are conducted on a workstation equipped with a 48GB NVIDIA RTX A6000 GPU and 9 CPU cores. We solve the CFS-based projection on full samples in our method through \texttt{jaxopt}.
\begin{figure}[tb!]
  \centering

  % % --- Top: SSR Heatmap ---
  % \begin{subfigure}[c]{0.49\textwidth}
  %   \centering
  %   \includegraphics[width=\linewidth]{figures/heatmap/ssr_heatmap.png}
  % \end{subfigure}

  % \par\vspace{5pt}

  % % --- Bottom row ---
  % \begin{subfigure}[c]{0.24\textwidth}
  %   \centering
  %   \includegraphics[width=\linewidth]{figures/heatmap/mean_cost_vs_level.png}
  % \end{subfigure}
  % \hfill
  % \begin{subfigure}[c]{0.24\textwidth}
  %   \centering
  %   \includegraphics[width=\linewidth]{figures/heatmap/all_levels_mean.png}
  % \end{subfigure}
  % \hfill
  % \begin{subfigure}[c]{0.24\textwidth}
  %   \centering
  %   \includegraphics[width=\linewidth]{figures/heatmap/time_per_level.png}
  % \end{subfigure}
  % \begin{subfigure}[c]{0.24\textwidth}
  %   \centering
  %   \includegraphics[width=0.7\linewidth]{figures/heatmap/constraint-legend.png}
  % \end{subfigure}

  % \includesvg[width=\linewidth]{figures/heatmap/ssr_heatmap_full}
    \includegraphics[width=\linewidth]{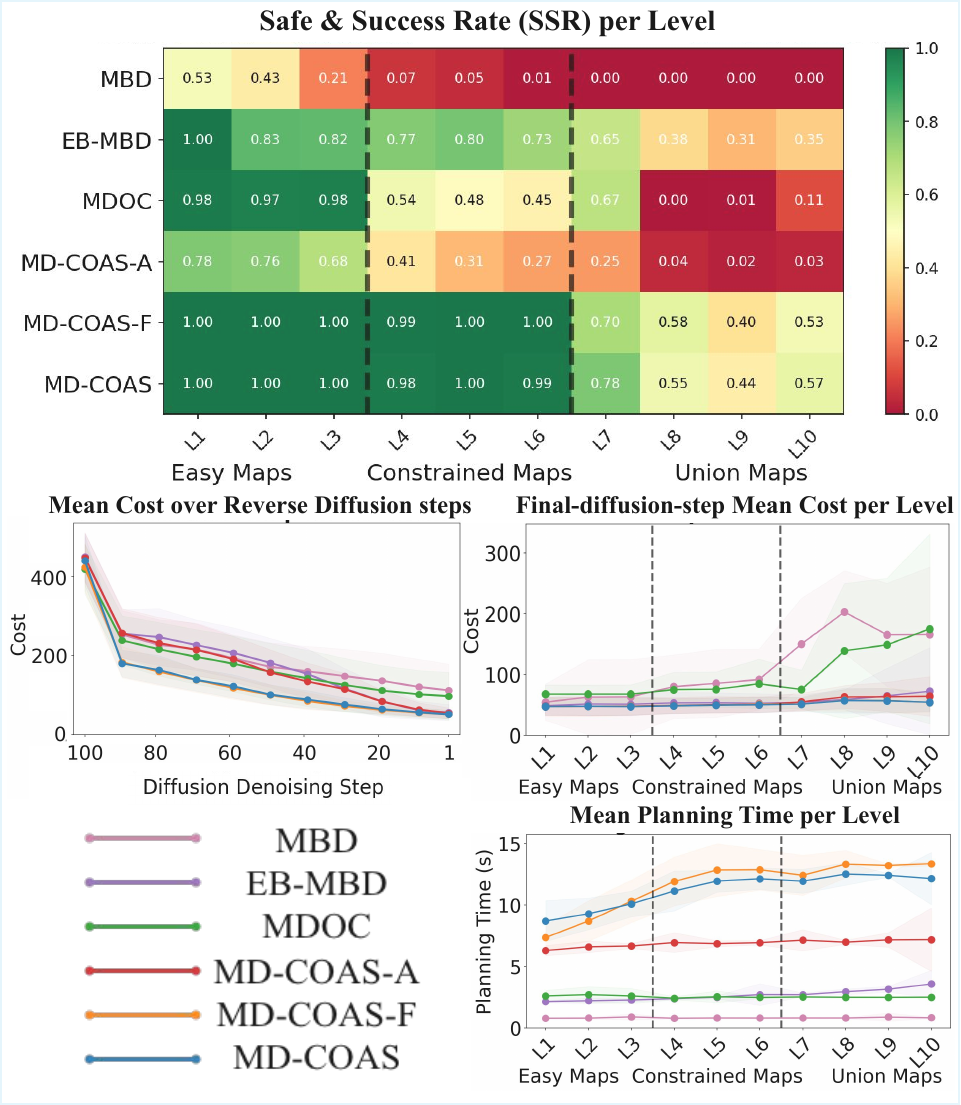}
  \caption{
  Overall performance comparison of model-based diffusion planners in the 2D constrained obstacle-avoidance benchmark.
  Top: Safe \& Success Rate (SSR) per Level  (avg. over seeds).
  Bottom: Mean Cost over Reverse Diffusion steps (avg. over levels and seeds); Final-diffusion-step Mean Cost per Level (avg. over seeds); Mean Planning Time per Level (avg. over seeds). Color regions indicate one standard deviation.
  }
  \label{fig:performance}
  \vspace{-0.5cm}
\end{figure}
\begin{figure}[tb!]
  \centering

  % % --- row of subfigures ---
  % \begin{subfigure}[t]{0.24\textwidth}
  % \textbf{Adaptive iALM Soft Prior}
  %   \centering
  %   \includegraphics[width=\linewidth]{figures/heatmap/all_levels_mean_r_k.png}
  % \end{subfigure}
  % \begin{subfigure}[t]{0.24\textwidth}
  % \textbf{Adaptive Hard QP Operator}
  %   \centering
  %   \includegraphics[width=\linewidth]{figures/heatmap/all_levels_mean_v_mean.png}
  % \end{subfigure}
  
  % \par\vspace{0pt}
  % \begin{subfigure}[t]{0.24\textwidth}
  %   \centering
  %   \includegraphics[width=\linewidth]{figures/heatmap/all_levels_mean_lambda.png}
  % \end{subfigure}
  % \begin{subfigure}[t]{0.24\textwidth}
  %   \centering
  %   \includegraphics[width=\linewidth]{figures/heatmap/all_levels_mean_p_k.png}
  % \end{subfigure}

  % \par\vspace{0pt}
  % \begin{subfigure}[t]{0.24\textwidth}
  %   \centering
  %   \includegraphics[width=\linewidth]{figures/heatmap/all_levels_mean_rho.png}
  % \end{subfigure}
  % \begin{subfigure}[t]{0.24\textwidth}
  % \centering
  %   \includegraphics[width=\linewidth]{figures/heatmap/all_levels_mean_I_QP.png}
  % \end{subfigure}

  % \par\vspace{0pt}
  % \begin{subfigure}[t]{0.24\textwidth}
  %   \centering
  %   \includegraphics[width=0.7\linewidth]{figures/heatmap/constraint-legend2.png}
  % \end{subfigure}
  % \begin{subfigure}[t]{0.24\textwidth}
  %   \centering
  %   \includegraphics[width=\linewidth]{figures/heatmap/all_levels_mean_topK.png}
  % \end{subfigure}
  \includegraphics[width=\linewidth]{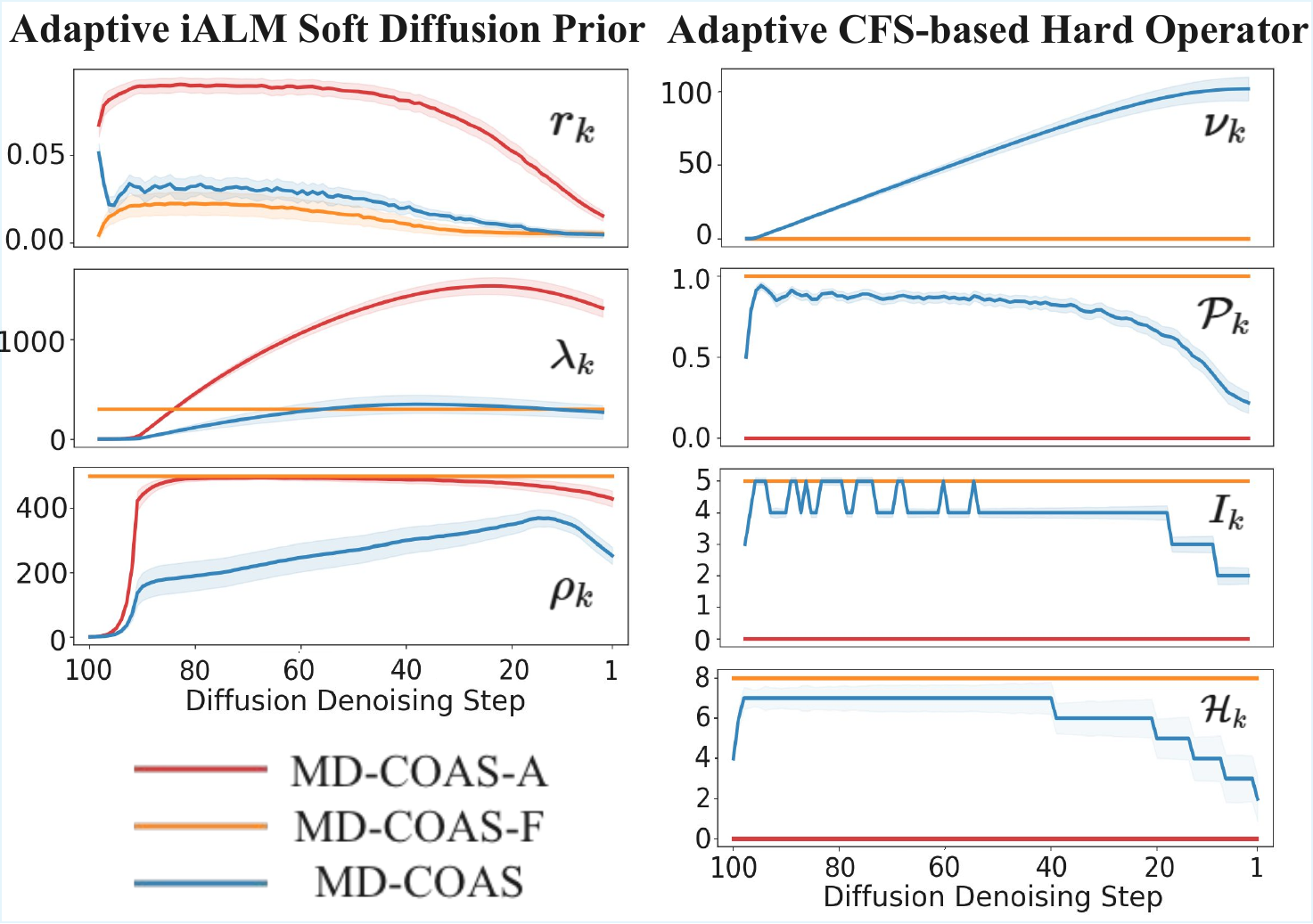}

  \caption{Adaptive scheduling of safety-enforcement variables over reverse diffusion steps in the 2D environment, averaged across all levels and seeds, for MD-COAS and two ablations. Left: residual-driven updates of the iALM dual variables $(\lambda_k,\rho_k)$ based on $r_k$. Right: scheduled CFS-based hard-projection parameters $(p_k, I_k, \mathcal{H}_k)$ and the compute dual $\nu_k$.}
  \label{fig:mechanism}
  \vspace{-0.5cm}
\end{figure}

\vspace{-0.2cm}
\subsection{2D Constrained Obstacle Avoidance}
\vspace{-0.1cm}

We evaluate MD-COAS against model-based baselines on a 2D constrained obstacle-avoidance benchmark. The maps are grouped into three families: (1) Easy Maps (L1--L3), which contain randomly placed circles and squares, with obstacle density increasing from approximately $2\%$ to $6\%$; (2) Constrained Maps (L4--L6), with more circles and squares, where obstacle density ranges from approximately $10\%$ to $16\%$; and (3) Union Maps (L7--L10), where primitive shapes are merged into random non-convex obstacles, with density increasing from approximately $18\%$ to $33\%$. For each level, we evaluate 10 random seeds. For each seed, each method generates 20 trajectory multimodals. The robot starts near the bottom region of the map and aims for a goal near the top region. The robot radius is fixed at $0.05$ in all experiments.

As illustrated in Fig.~\ref{fig:performance}, we report the Safe \& Success Rate (SSR) at different levels. MD-COAS and its variant MD-COAS-F consistently achieve 1.00 SSR in all Easy and Constrained Maps. In contrast, other baselines degrade in Constrained Maps, with EB-MBD dropping to approximately 0.7 and MDOC to approximately 0.5. For the highly non-convex Union Maps, MD-COAS and MD-COAS-F maintain approximately 0.5 SSR, outperforming other baselines under the same settings. The cost maps further show that MD-COAS achieves the lowest planning cost across levels and exhibits the fastest convergence during the diffusion denoising process.

\begin{figure*}[tb!]
  \centering

    \caption{7-DoF arm obstacle-avoidance task in D3IL. Left: environment overview with a representative final trajectory produced by MD-COAS. Right Three: final planned end-effector trajectories over five seeds (100 trajectory multimodals in total) as the obstacle radius increases ($0.03\!\rightarrow\!0.04\!\rightarrow\!0.05$), visualized on the 2D workspace plane.}
  \begin{subfigure}[b]{0.27\textwidth}
    \centering
    \includegraphics[width=\linewidth]{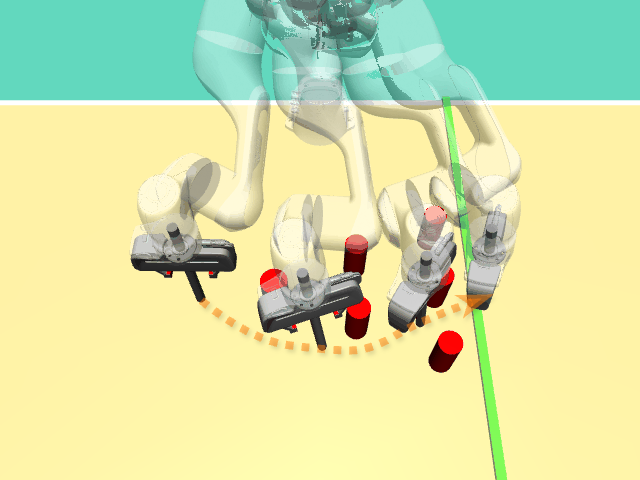}
  \end{subfigure}\hfill
  \begin{subfigure}[b]{0.17\textwidth}
    \centering
    \includegraphics[width=\linewidth]{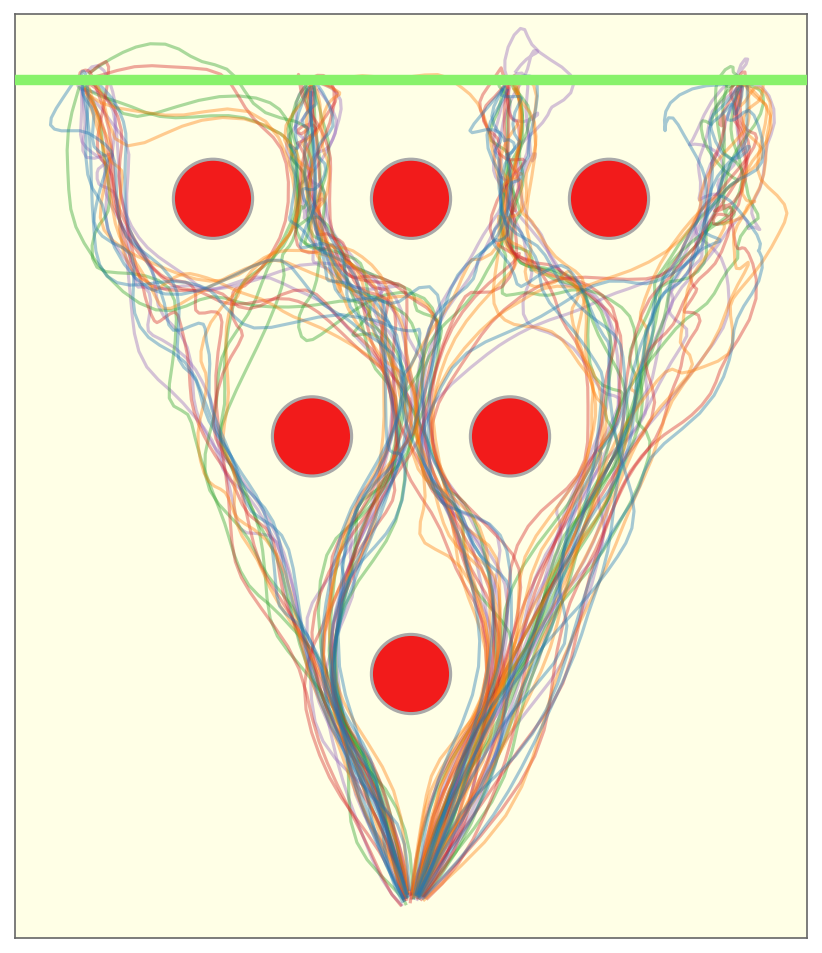}
  \end{subfigure}\hfill
  \begin{subfigure}[b]{0.17\textwidth}
    \centering
    \includegraphics[width=\linewidth]{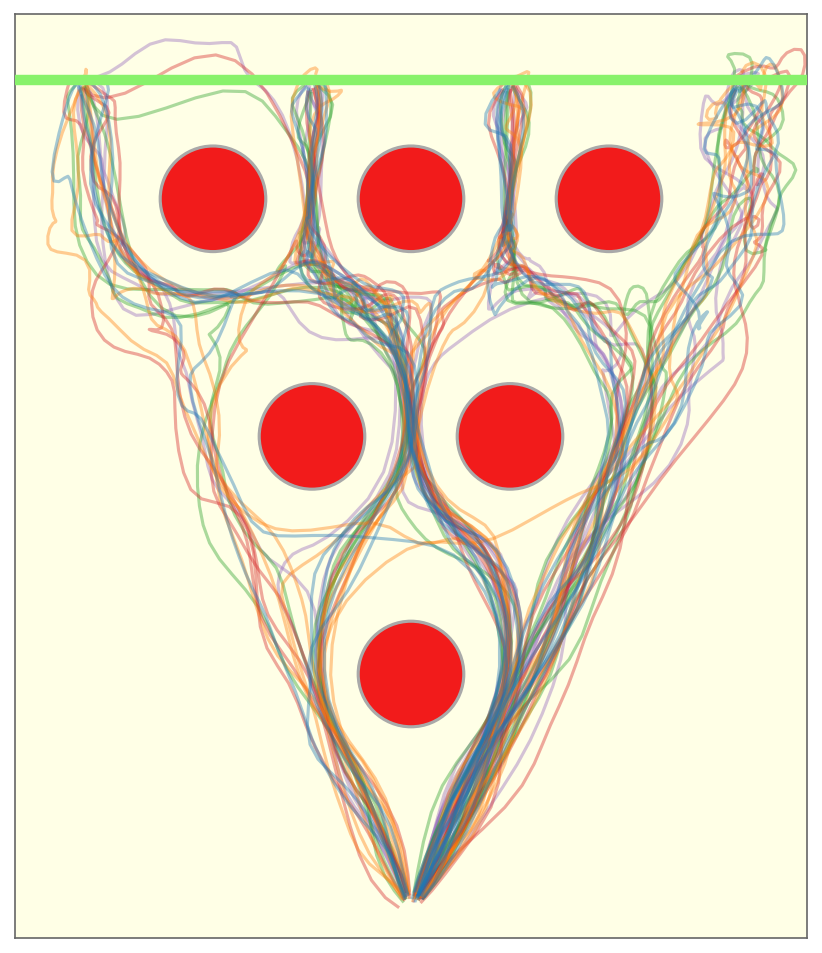}
  \end{subfigure}\hfill
  \begin{subfigure}[b]{0.17\textwidth}
    \centering
    \includegraphics[width=\linewidth]{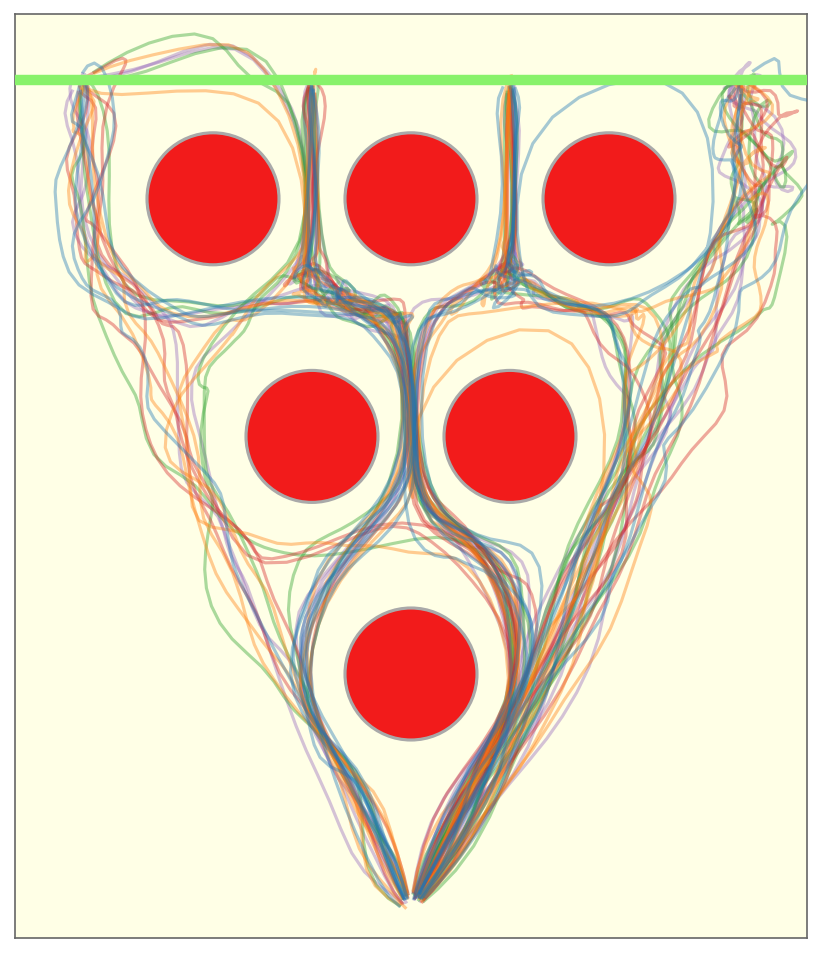}
  \end{subfigure}\hfill
  \begin{subfigure}[b]{0.15\textwidth}
    \centering
    \includegraphics[width=\linewidth]{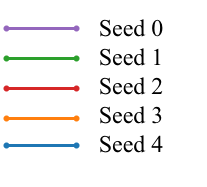}
  \end{subfigure}
  \label{fig:exp2}
\end{figure*}

\begin{table*}[tb!]
\centering
\small
\setlength{\tabcolsep}{2pt}
\renewcommand{\arraystretch}{1}

\begin{tabular}{c|cccc|cccc|cccc}
\toprule
& \multicolumn{4}{c|}{Obstacle Radius$=0.03$} 
& \multicolumn{4}{c|}{Obstacle Radius$=0.04$}
& \multicolumn{4}{c}{Obstacle Radius$=0.05$} \\
\cmidrule(lr){2-5} \cmidrule(lr){6-9} \cmidrule(lr){10-13}
Method 
& \textbf{SSR}($\uparrow$) & \textbf{V}($\downarrow$) & \textbf{T}($\downarrow$) & \textbf{P}($\downarrow$)
& \textbf{SSR}($\uparrow$) & \textbf{V}($\downarrow$) & \textbf{T}($\downarrow$) & \textbf{P}($\downarrow$)
& \textbf{SSR}($\uparrow$) & \textbf{V}($\downarrow$) & \textbf{T}($\downarrow$) & \textbf{P}($\downarrow$) \\
\midrule

SafeDiffuser~\cite{xiao2023safediffuser}
& 0.45 & 1.30 & \cellcolor{orange!25}5.49$^{0.20}$ & 1.57$^{0.05}$
& 0.35 & 2.68 & 11.53$^{1.30}$ & 1.56$^{0.05}$
& 0.30 & 4.40 & 11.18$^{0.10}$ & 1.54$^{0.03}$
\\

DPCC-C~\cite{romer2024diffusion}
& 0.75 & \cellcolor{orange!25}0.78 & 1528.91$^{571.70}$ & 1.32$^{0.39}$
& 0.54 & 1.44 & 2561.96$^{1818.70}$ & 1.51$^{0.53}$
& 0.64 & 2.10 & 3460.31$^{1568.70}$ & 1.44$^{0.40}$
\\

EB-MBD~\cite{mishra2025eb} 
& \cellcolor{orange!25}0.98 
& \cellcolor{red!25}\textbf{0.00} 
& \cellcolor{red!25}\textbf{4.37}$^{2.00}$ 
& 1.01$^{0.05}$
& 0.91 
& \cellcolor{red!25}\textbf{0.00} 
& \cellcolor{red!25}\textbf{4.02}$^{0.60}$ 
& 1.02$^{0.05}$
& \cellcolor{orange!25}0.82 
& \cellcolor{red!25}\textbf{0.00} 
& \cellcolor{red!25}\textbf{4.06}$^{0.60}$ 
& 1.01$^{0.06}$
\\

MDOC~\cite{he2026}
& \cellcolor{red!25}\textbf{1.00} 
& \cellcolor{red!25}\textbf{0.00} 
& 5.65$^{1.90}$ 
& \cellcolor{red!25}\textbf{0.89}$^{0.04}$
& \cellcolor{orange!25}0.99 
& \cellcolor{orange!25}0.04 
& \cellcolor{orange!25}4.90$^{0.60}$ 
& \cellcolor{red!25}\textbf{0.88}$^{0.04}$
& 0.73 
& \cellcolor{orange!25}0.02 
& \cellcolor{orange!25}4.87$^{0.70}$ 
& \cellcolor{red!25}\textbf{0.90}$^{0.04}$
\\

MD-COAS (Ours)
& \cellcolor{red!25}\textbf{1.00} 
& \cellcolor{red!25}\textbf{0.00} 
& 13.53$^{2.90}$ 
& \cellcolor{orange!25}0.99$^{0.04}$
& \cellcolor{red!25}\textbf{1.00} 
& \cellcolor{red!25}\textbf{0.00} 
& 13.98$^{2.10}$ 
& \cellcolor{orange!25}0.99$^{0.05}$
& \cellcolor{red!25}\textbf{1.00} 
& \cellcolor{red!25}\textbf{0.00} 
& 14.87$^{2.10}$ 
& \cellcolor{orange!25}1.00$^{0.07}$
\\
\bottomrule
\end{tabular}
\caption{Quantitative results for 7-DoF arm obstacle avoidance as obstacle radius increases ($0.03\!\rightarrow\!0.04\!\rightarrow\!0.05$). Columns report \textbf{SSR}($\uparrow$) (Safe \& Success Rate), \textbf{V}($\downarrow$) (violation rate, \%), \textbf{T}($\downarrow$) (planning time), and \textbf{P}($\downarrow$) (path length). Values are means across all seeds, and path-length means over safe \& success trajectories, with standard deviations given as right superscripts. Note that CFS projection is applied to all unweighted trajectory samples, introducing a trade-off between planning performance and computation time. Nevertheless, our method still achieves a substantial planning-time advantage over DPCC-C, which also performs batch sample projection but does not fully exploit GPU parallelization.}
\label{tab:exp2}
\vspace{-0.4cm}
\end{table*}

To better interpret the denoising dynamics, we visualize samples at 90\%, 50\%, and 10\% of the reverse process at L10 (Fig.~\ref{fig:diffusion_process}). MBD continues to generate trajectories that penetrate obstacles in late denoising due to the lack of safety enforcement. EB-MBD uses a fixed soft barrier, which cannot adjust its parameters to the varying violation patterns across reverse steps, so most final samples still violate constraints. MDOC enforces safety through CBF projections, yet the induced feasible set is overly restrictive, slowing reverse convergence (trajectories remain partially refined around 50\%), failing reaching target, and collapsing multimodality into one final mode in a constrained environment. In contrast, MD-COAS-F and MD-COAS maintain strong safety with faster convergence while preserving a larger effective feasible region during denoising, enabling broader multimodal exploration and producing diverse high-quality trajectories.

\textbf{Why constraint optimization (CO)?}
Fig.~\ref{fig:performance} shows that model-based diffusion baselines struggle to sustain high SSR as constraint nonconvexity increases: their evaluation costs remain high at deeper difficulty levels, indicating frequent collisions or conservative, suboptimal trajectories. Among our variants, MD-COAS-A converges but consistently exhibits a higher terminal cost floor and reduced SSR in tight, highly nonconvex maps, suggesting that soft feasibility weighting alone may stall in more constrained environments and cannot ensure strict collision avoidance. In contrast, adding the CFS-based hard projection (MD-COAS-F and MD-COAS) provides an explicit feasibility correction that pulls proposals back into a feasible set, yielding lower costs and higher SSR. 

\textbf{Why adaptive scheduling (AS)?}
While hard projection improves robustness, applying it uniformly is computationally expensive. Fig.~\ref{fig:performance} indicates that MD-COAS-F attains strong safety and low cost but pays a higher runtime cost when the environment becomes more constrained, whereas the MD-COAS achieves comparable terminal performance with improved efficiency. The reason is that AS treats safety enforcement as a controllable computing resource: it increases feasibility pressure when violations persist via soft-prior updates and invokes hard projection only when needed under a budget via scheduled $(p_k,I_k,\mathcal{H}_k)$. As in Fig.~\ref{fig:mechanism}, MD-COAS dynamically modulates these variables across diffusion steps, concentrating the projection effort on the steps when most needed, thereby improving runtime without sacrificing final solution quality.

\subsection{D3IL 7-DoF Arm Avoidance}
\vspace{-0.2cm}
We further evaluate the robustness of MD-COAS in a higher-dimensional 7-DoF robotic arm avoidance task built upon D3IL~\cite{jia2024towards}. Unlike the 2D benchmark, which performs planning in a two-dimensional control space, this setting requires trajectory optimization in the 7-dimensional joint-space control of the robotic arm. Although the end-effector motion remains constrained to a 2D plane, its position is determined by joint variables through a nonlinear kinematic mapping, making the system dynamics non-linear. The CFS-based projection operator is directly applied in the 7-DoF joint space. The robotic arm must avoid cylindrical obstacles (shown in red) and reach a target line (green), as illustrated in Fig.~\ref{fig:exp2}. The obstacle radius is varied across three difficulty levels (0.03, 0.04, and 0.05), while the end-effector radius is fixed at 0.01. For each obstacle level, we evaluated performance on 5 random seeds, with 20 trajectory multimodals conducted per seed.

Table~\ref{tab:exp2} reports results with increasing obstacle radius. As the obstacle size increases, the feasible region progressively shrinks, making safe trajectory generation more challenging.
Overall, the observed trends are consistent with those in the 2D benchmark. Both SafeDiffuser and DPCC-C show reduced robustness as task difficulty increases. EB-MBD maintains zero violations but exhibits a gradual decline in success rate under tighter constraints. MDOC performs competitively at lower difficulty levels, but exhibits both minor violations and a noticeable drop in success rate under tighter constraints. This behavior is further reflected in its path characteristics. The shorter paths of MDOC are attributable to its tighter feasible set, which enforces more deterministic behavior within each trajectory, whereas MD-COAS preserves greater stochasticity to maintain trajectory diversity, occasionally introducing mild detours and slightly increasing path length. In contrast, MD-COAS consistently achieves 100\% success with zero violations at all obstacle levels, maintaining stable performance as task difficulty increases. These results align with the observations of the 2D Constrained Obstacle Avoidance experiment and further demonstrate the robustness of the proposed method in higher-dimensional settings. 

\section{Conclusion}
% \vspace{-0.2cm}
In this work, we propose MD-COAS, a constraint optimization framework that combines an iALM soft diffusion prior and a CFS-based hard projection operator. We also develop constraint scheduling to adaptively co-optimize for both soft diffusion prior and CFS hard projection computation, along with diffusion reverse scheduling. The results show the superior performance of our method in highly constrained 2D and 7-DoF environments. Future directions include improving sampling projection efficiency, adapting the planner to higher-dimensional tasks with richer dynamics models, and exploring under multi-robot coordination settings.

\bibliographystyle{IEEEtran}
\bibliography{references}

@inproceedings{karaman2011anytime,
  title={Anytime motion planning using the RRT},
  author={Karaman, Sertac and Walter, Matthew R and Perez, Alejandro and Frazzoli, Emilio and Teller, Seth},
  booktitle={International Conference on Robotics and Automation},
  pages={1478--1483},
  year={2011},
  organization={IEEE}
}

@inproceedings{kuffner2000rrt,
  title={RRT-connect: An efficient approach to single-query path planning},
  author={Kuffner, James J and LaValle, Steven M},
  booktitle={Proceedings 2000 ICRA. Millennium conference. IEEE international conference on robotics and automation. Symposia proceedings (Cat. No. 00CH37065)},
  volume={2},
  pages={995--1001},
  year={2000},
  organization={IEEE}
}

@inproceedings{xiao2023safediffuser,
  title={Safediffuser: Safe planning with diffusion probabilistic models},
  author={Xiao, Wei and Wang, Tsun-Hsuan and Gan, Chuang and Hasani, Ramin and Lechner, Mathias and Rus, Daniela},
  booktitle={International Conference on Learning Representations},
  year={2023}
}

@inproceedings{carvalho2023motion,
  title={Motion planning diffusion: Learning and planning of robot motions with diffusion models},
  author={Carvalho, Joao and Le, An T and Baierl, Mark and Koert, Dorothea and Peters, Jan},
  booktitle={International Conference on Intelligent Robots and Systems},
  pages={1916--1923},
  year={2023},
  organization={IEEE}
}

@inproceedings{bonalli2019gusto,
  title={Gusto: Guaranteed sequential trajectory optimization via sequential convex programming},
  author={Bonalli, Riccardo and Cauligi, Abhishek and Bylard, Andrew and Pavone, Marco},
  booktitle={International Conference on Robotics and Automation},
  pages={6741--6747},
  year={2019},
  organization={IEEE}
}

@inproceedings{howell2019altro,
  title={ALTRO: A fast solver for constrained trajectory optimization},
  author={Howell, Taylor A and Jackson, Brian E and Manchester, Zachary},
  booktitle={International Conference on Intelligent Robots and Systems},
  pages={7674--7679},
  year={2019},
  organization={IEEE}
}

@inproceedings{jallet2022constrained,
  title={Constrained differential dynamic programming: A primal-dual augmented lagrangian approach},
  author={Jallet, Wilson and Bambade, Antoine and Mansard, Nicolas and Carpentier, Justin},
  booktitle={International Conference on Intelligent Robots and Systems},
  pages={13371--13378},
  year={2022},
  organization={IEEE}
}

@inproceedings{zhou2021distributed,
  title={Distributed motion coordination using convex feasible set based model predictive control},
  author={Zhou, Hongyu and Liu, Changliu},
  booktitle={International Conference on Robotics and Automation},
  pages={8330--8336},
  year={2021},
  organization={IEEE}
}

@inproceedings{peters2010relative,
  title={Relative entropy policy search},
  author={Peters, Jan and Mulling, Katharina and Altun, Yasemin},
  booktitle={AAAI conference on artificial intelligence},
  volume={24},
  number={1},
  pages={1607--1612},
  year={2010}
}

@article{kavraki1996probabilistic,
  title={Probabilistic roadmaps for path planning in high-dimensional configuration spaces},
  author={Kavraki, Lydia E and Svestka, Petr and Latombe, J-C and Overmars, Mark H},
  journal={IEEE transactions on Robotics and Automation},
  volume={12},
  number={4},
  pages={566--580},
  year={1996},
  publisher={IEEE}
}

@article{la2011motion,
  title={Motion planning},
  author={La Valle, Steven M},
  journal={IEEE Robotics \& Automation Magazine},
  volume={18},
  number={2},
  pages={108--118},
  year={2011},
  publisher={IEEE}
}

@article{kleinbort2018probabilistic,
  title={Probabilistic completeness of RRT for geometric and kinodynamic planning with forward propagation},
  author={Kleinbort, Michal and Solovey, Kiril and Littlefield, Zakary and Bekris, Kostas E and Halperin, Dan},
  journal={IEEE Robotics and Automation Letters},
  volume={4},
  number={2},
  pages={i--vii},
  year={2018},
  publisher={IEEE}
}

@article{williams2018information,
  title={Information-theoretic model predictive control: Theory and applications to autonomous driving},
  author={Williams, Grady and Drews, Paul and Goldfain, Brian and Rehg, James M and Theodorou, Evangelos A},
  journal={IEEE Transactions on Robotics},
  volume={34},
  number={6},
  pages={1603--1622},
  year={2018},
  publisher={IEEE}
}

@article{chi2025diffusion,
  title={Diffusion policy: Visuomotor policy learning via action diffusion},
  author={Chi, Cheng and Xu, Zhenjia and Feng, Siyuan and Cousineau, Eric and Du, Yilun and Burchfiel, Benjamin and Tedrake, Russ and Song, Shuran},
  journal={The International Journal of Robotics Research},
  volume={44},
  number={10-11},
  pages={1684--1704},
  year={2025},
  publisher={Sage Publications Sage UK: London, England}
}

@article{liu2018convex,
  title={The convex feasible set algorithm for real time optimization in motion planning},
  author={Liu, Changliu and Lin, Chung-Yen and Tomizuka, Masayoshi},
  journal={SIAM Journal on Control and optimization},
  volume={56},
  number={4},
  pages={2712--2733},
  year={2018},
  publisher={SIAM}
}

@article{schulman2014motion,
  title={Motion planning with sequential convex optimization and convex collision checking},
  author={Schulman, John and Duan, Yan and Ho, Jonathan and Lee, Alex and Awwal, Ibrahim and Bradlow, Henry and Pan, Jia and Patil, Sachin and Goldberg, Ken and Abbeel, Pieter},
  journal={The International Journal of Robotics Research},
  volume={33},
  number={9},
  pages={1251--1270},
  year={2014},
  publisher={Sage Publications Sage UK: London, England}
}

@article{christopher2024constrained,
  title={Constrained synthesis with projected diffusion models},
  author={Christopher, Jacob K and Baek, Stephen and Fioretto, Nando},
  journal={Advances in Neural Information Processing Systems},
  volume={37},
  pages={89307--89333},
  year={2024}
}

@article{pan2024model,
  title={Model-based diffusion for trajectory optimization},
  author={Pan, Chaoyi and Yi, Zeji and Shi, Guanya and Qu, Guannan},
  journal={Advances in Neural Information Processing Systems},
  volume={37},
  pages={57914--57943},
  year={2024}
}

@article{ames2016control,
  title={Control barrier function based quadratic programs for safety critical systems},
  author={Ames, Aaron D and Xu, Xiangru and Grizzle, Jessy W and Tabuada, Paulo},
  journal={IEEE Transactions on Automatic Control},
  volume={62},
  number={8},
  pages={3861--3876},
  year={2016},
  publisher={IEEE}
}

@article{ho2020denoising,
  title={Denoising diffusion probabilistic models},
  author={Ho, Jonathan and Jain, Ajay and Abbeel, Pieter},
  journal={Advances in Neural Information Processing Systems},
  volume={33},
  pages={6840--6851},
  year={2020}
}

@article{sahin2019inexact,
  title={An inexact augmented Lagrangian framework for nonconvex optimization with nonlinear constraints},
  author={Sahin, Mehmet Fatih and Alacaoglu, Ahmet and Latorre, Fabian and Cevher, Volkan and others},
  journal={Advances in Neural Information Processing Systems},
  volume={32},
  year={2019}
}

@article{xu2017admm,
  title={ADMM without a fixed penalty parameter: Faster convergence with new adaptive penalization},
  author={Xu, Yi and Liu, Mingrui and Lin, Qihang and Yang, Tianbao},
  journal={Advances in Neural Information Processing Systems},
  volume={30},
  year={2017}
}

@article{jia2024towards,
  title={Towards diverse behaviors: A benchmark for imitation learning with human demonstrations},
  author={Jia, Xiaogang and Blessing, Denis and Jiang, Xinkai and Reuss, Moritz and Donat, Atalay and Lioutikov, Rudolf and Neumann, Gerhard},
  journal={arXiv preprint arXiv:2402.14606},
  year={2024}
}

@article{mishra2025eb,
  title={EB-MBD: Emerging-Barrier Model-Based Diffusion for Safe Trajectory Optimization in Highly Constrained Environments},
  author={Mishra, Raghav and Manchester, Ian R},
  journal={arXiv preprint arXiv:2510.07700},
  year={2025}
}

@article{romer2024diffusion,
  title={Diffusion predictive control with constraints},
  author={R{\"o}mer, Ralf and von Rohr, Alexander and Schoellig, Angela P},
  journal={arXiv preprint arXiv:2412.09342},
  year={2024}
}

@article{janner2022planning,
  title={Planning with diffusion for flexible behavior synthesis},
  author={Janner, Michael and Du, Yilun and Tenenbaum, Joshua B and Levine, Sergey},
  journal={arXiv preprint arXiv:2205.09991},
  year={2022}
}

@article{zhang2025constrained,
  title={Constrained diffusers for safe planning and control},
  author={Zhang, Jichen and Zhao, Liqun and Papachristodoulou, Antonis and Umenberger, Jack},
  journal={arXiv preprint arXiv:2506.12544},
  year={2025}
}

@article{liang2025simultaneous,
  title={Simultaneous multi-robot motion planning with projected diffusion models},
  author={Liang, Jinhao and Christopher, Jacob K and Koenig, Sven and Fioretto, Ferdinando},
  journal={arXiv preprint arXiv:2502.03607},
  year={2025}
}

@article{he2026,
  title={Model-Based Diffusion Optimal Control for Multi-Robot Motion Planning},
  author={Zhilin He and Yorai Shaoul and Jiaoyang Li},
  journal={arXiv preprint arXiv:2607.12423},
  year={2026}
}

@article{liang2024multi,
  title={Multi-agent path finding in continuous spaces with projected diffusion models},
  author={Liang, Jinhao and Christopher, Jacob K and Koenig, Sven and Fioretto, Ferdinando},
  journal={arXiv preprint arXiv:2412.17993},
  year={2024}
}

@article{cheng2025safe,
  title={Safe and stable control via lyapunov-guided diffusion models},
  author={Cheng, Xiaoyuan and Tang, Xiaohang and Yang, Yiming},
  journal={arXiv preprint arXiv:2509.25375},
  year={2025}
}

@article{fishman2023diffusion,
  title={Diffusion models for constrained domains},
  author={Fishman, Nic and Klarner, Leo and De Bortoli, Valentin and Mathieu, Emile and Hutchinson, Michael},
  journal={arXiv preprint arXiv:2304.05364},
  year={2023}
}

@article{levine2018reinforcement,
  title={Reinforcement learning and control as probabilistic inference: Tutorial and review},
  author={Levine, Sergey},
  journal={arXiv preprint arXiv:1805.00909},
  year={2018}
}

@article{wohlberg2017admm,
  title={ADMM penalty parameter selection by residual balancing},
  author={Wohlberg, Brendt},
  journal={arXiv preprint arXiv:1704.06209},
  year={2017}
}

@book{wright1997primal,
  title={Primal-dual interior-point methods},
  author={Wright, Stephen J},
  year={1997},
  publisher={SIAM}
}

@book{betts2010practical,
  title={Practical methods for optimal control and estimation using nonlinear programming},
  author={Betts, John T},
  year={2010},
  publisher={SIAM}
}

@incollection{botev2013cross,
  title={The cross-entropy method for optimization},
  author={Botev, Zdravko I and Kroese, Dirk P and Rubinstein, Reuven Y and L’ecuyer, Pierre},
  booktitle={Handbook of statistics},
  volume={31},
  pages={35--59},
  year={2013},
  publisher={Elsevier}
}

\clearpage
\appendices
\section{Theoretical Analysis}
\label{app:theory}

\subsection{Proof Objective}
This appendix formalizes the safety-enforcement mechanism used by
MD-COAS in the main text. The goal is to show that, under standard
local regularity assumptions, the iALM prior in
\eqref{eq:alm_objective}--\eqref{eq:pk_alm} provides an informative
soft feasibility signal, the CFS-QP in \eqref{eq:cfs_qp_full} gives a
local hard correction for the convexified constraints, and the
batch-level gate in \eqref{eq:qp_gate} yields an expected residual
recursion whose contraction strength is controlled by the scheduled
projection probability $p_k$.

At reverse step $k$, let
$\mathcal B_k=\{\tau_{k,m}\}_{m=1}^{M}$ denote the rollout batch. For
each trajectory, define the same horizon-normalized violation used by
the main method,
\begin{equation}
\bar g(\tau_{k,m})
=
\frac{1}{T}\sum_{t=1}^{T}
[g^t(s_{k,m}^t,u_{k,m}^t)]_+ .
\label{eq:app_gbar}
\end{equation}
We track the same robust batch residual as
\eqref{eq:primal_residual},
\begin{equation}
R(\mathcal B_k)
=
\operatorname{Quantile}_{0.9}
\!\left(
\{\|\bar g(\tau_{k,m})\|_2\}_{m=1}^{M}
\right).
\label{eq:app_batch_residual}
\end{equation}
The iALM energy corresponding to
\eqref{eq:alm_objective} is
\begin{equation}
\Phi_k(\tau)
=
\mathcal J(\tau)
+\lambda_k^\top \bar g(\tau)
+\frac{\rho_k}{2}\|\bar g(\tau)\|_2^2 .
\label{eq:app_ialm_energy}
\end{equation}
The induced sampling weights are therefore proportional to
$\exp(-\Phi_k(\tau)/\beta)$, matching the target distribution in
\eqref{eq:pk_alm}.

\subsection{iALM Soft Feasibility Weighting}

\begin{lemma}[Monotonic soft feasibility weighting]
\label{lem:app_ialm_weight}
Let $\lambda_k\ge0$, $\rho_k>0$, and $\beta>0$. For any two rollouts
$\tau_a$ and $\tau_b$ with $\mathcal J(\tau_a)=\mathcal J(\tau_b)$ and
$\bar g(\tau_a)\succeq \bar g(\tau_b)\succeq0$, the iALM target in
\eqref{eq:app_ialm_energy} assigns no larger weight to $\tau_a$ than to
$\tau_b$. If at least one violated component is strictly larger, then
the weight of $\tau_a$ is strictly smaller.
\end{lemma}

\begin{proof}
The unnormalized iALM weight is
\[
W_k(\tau)=\exp(-\Phi_k(\tau)/\beta).
\]
Since the two rollouts have the same task cost, their log-weight ratio
is
\begin{equation}
\begin{aligned}
\log\frac{W_k(\tau_a)}{W_k(\tau_b)}
&=
-\frac{1}{\beta}
\lambda_k^\top
\bigl(\bar g(\tau_a)-\bar g(\tau_b)\bigr)\\
&\quad
-\frac{\rho_k}{2\beta}
\left(
\|\bar g(\tau_a)\|_2^2
-\|\bar g(\tau_b)\|_2^2
\right).
\end{aligned}
\label{eq:app_weight_ratio}
\end{equation}
Because
$\bar g(\tau_a)\succeq \bar g(\tau_b)\succeq0$ and
$\lambda_k\ge0$, the linear term in
\eqref{eq:app_weight_ratio} is non-positive. The norm difference is
also non-negative, so the quadratic term is non-positive. Therefore
$W_k(\tau_a)\le W_k(\tau_b)$. If a violated component is strictly
larger, then the quadratic norm strictly increases; since $\rho_k>0$,
the second term is strictly negative and
$W_k(\tau_a)<W_k(\tau_b)$. Thus the iALM prior gives a continuous,
nonzero, and monotone feasibility signal to the MCSA weights rather
than discarding infeasible samples through a hard indicator.
\end{proof}

\subsection{Assumptions}

\begin{assumption}[Regular rollout model]
\label{ass:app_rollout}
On the compact domain reached by the planner, the rollout map from the
stacked control sequence to the trajectory is continuously
differentiable. The task cost $\mathcal J$ and every component of
$\bar g$ are locally Lipschitz and have bounded first derivatives.
\end{assumption}

\begin{assumption}[Inexact MCSA proposal]
\label{ass:app_mcsa}
Conditioned on the history $\mathcal F_k$, the MBD/MCSA proposal batch
$\tilde{\mathcal B}_k=\{\tilde\tau_{k,m}\}_{m=1}^{M}$ satisfies
\begin{equation}
\mathbb E
\!\left[
R(\tilde{\mathcal B}_k)\mid\mathcal F_k
\right]
\le
R(\mathcal B_k)+\delta_{m,k},
\label{eq:app_mcsa_error}
\end{equation}
where $\delta_{m,k}\ge0$ captures finite-sample MCSA error, diffusion
temperature bias, and numerical rollout error.
\end{assumption}

\begin{assumption}[Local CFS correction]
\label{ass:app_cfs_contract}
For the scheduled CFS operator
$\mathcal P_{I_k,\mathcal H_k}$ in \eqref{eq:cfs_qp_full}, the selected
obstacle-time half-spaces are valid local inner approximations of the
safe set, and the corresponding convex QPs are feasible. There exist
$\alpha_k\in(0,1]$ and $\delta_{q,k}\ge0$ such that
\begin{equation}
\begin{aligned}
&
\mathbb E
\!\left[
R\!\left(
\mathcal P_{I_k,\mathcal H_k}
(\tilde{\mathcal B}_k)
\right)
\mid\mathcal F_k
\right] \\
&\qquad\le
(1-\alpha_k)
\mathbb E
\!\left[
R(\tilde{\mathcal B}_k)\mid\mathcal F_k
\right]
+\delta_{q,k}.
\end{aligned}
\label{eq:app_cfs_error}
\end{equation}
The term $\delta_{q,k}$ is the inexactness floor from finite CFS
iterations, finite QP tolerance, active-set truncation, and rollout
linearization.
\end{assumption}

\subsection{CFS Local Hard Correction}

\begin{lemma}[Feasibility of the CFS-QP correction]
\label{lem:app_cfs_local}
Consider one nominal trajectory $\tilde\tau_{k,m}$ and the CFS-QP in
\eqref{eq:cfs_qp_full}. If the half-spaces in
\eqref{eq:cfs_halfspace} are valid inner approximations and the QP is
solved exactly, then the corrected linearized rollout satisfies every
selected convexified collision constraint in $\mathcal A_{k,m}$. If the
QP is solved with primal tolerance $\epsilon_{\mathrm{qp}}$, the
violation of each selected linearized constraint is at most
$\epsilon_{\mathrm{qp}}$. If the local CFS model error at the corrected
control is bounded by $e_{\mathrm{cfs}}^t$, then the corresponding
nonlinear signed-distance violation is bounded by
$\epsilon_{\mathrm{qp}}+e_{\mathrm{cfs}}^t$.
\end{lemma}

\begin{proof}
For a selected pair $(j,t)\in\mathcal A_{k,m}$, CFS constructs the
half-space
\begin{equation}
\mathcal F_j^t
=
\left\{
q:\,
\phi_j(q_{k,m}^t)
+\hat\nabla\phi_j(q_{k,m}^t)^\top
(q-q_{k,m}^t)
\ge0
\right\}.
\label{eq:app_cfs_halfspace_short}
\end{equation}
By the local-validity condition in Assumption~\ref{ass:app_cfs_contract},
$\mathcal F_j^t$ lies inside the local safe side of obstacle $j$. The
rollout linearization in \eqref{eq:rollout_lin} defines the affine
prediction
\begin{equation}
\ell^t(u)
=
q_{k,m}^t
+J_{q,k,m}^t(u-u_{k,m}).
\label{eq:app_rollout_lin_short}
\end{equation}
Substituting \eqref{eq:app_rollout_lin_short} into
\eqref{eq:app_cfs_halfspace_short} yields the linear control-space
constraint
\begin{equation}
(a_{j,k,m}^t)^\top u\ge b_{j,k,m}^t,
\label{eq:app_cfs_row}
\end{equation}
with $a_{j,k,m}^t$ and $b_{j,k,m}^t$ defined as in
\eqref{eq:cfs_linear_u}. The QP in \eqref{eq:cfs_qp_full} optimizes
over the intersection of all such rows. Hence, an exact solution
$u^\star$ satisfies \eqref{eq:app_cfs_row} for every
$(j,t)\in\mathcal A_{k,m}$, so the corresponding linearized workspace
point belongs to every selected half-space. Since each selected
half-space is a local inner approximation, the corrected linearized
rollout is locally safe for all selected constraints. If the solver
returns an inexact solution satisfying
\begin{equation}
(a_{j,k,m}^t)^\top u^\star
\ge
b_{j,k,m}^t-\epsilon_{\mathrm{qp}},
\label{eq:app_qp_tol}
\end{equation}
then the selected linearized constraint violation is
\begin{equation}
\left[
b_{j,k,m}^t
-(a_{j,k,m}^t)^\top u^\star
\right]_+
\le
\epsilon_{\mathrm{qp}},
\end{equation}
which proves the linearized tolerance statement. For the nonlinear
rollout, define the local CFS model error
\begin{equation}
\begin{aligned}
e_{\mathrm{cfs}}^t
&\ge
\Bigl|
\phi_j(q^t(u^\star))
-\phi_j(q_{k,m}^t)\\
&\quad
-\hat\nabla\phi_j(q_{k,m}^t)^\top
(\ell^t(u^\star)-q_{k,m}^t)
\Bigr|.
\end{aligned}
\label{eq:app_cfs_model_error}
\end{equation}
Combining \eqref{eq:app_cfs_model_error} with the
$\epsilon_{\mathrm{qp}}$-feasible half-space constraint yields a
nonlinear violation no larger than
$\epsilon_{\mathrm{qp}}+e_{\mathrm{cfs}}^t$. The accumulated effect of
this local model error, finite iterations, and active-set truncation is
the projection floor $\delta_{q,k}$ used in
Assumption~\ref{ass:app_cfs_contract}.
\end{proof}

\subsection{Batch-Gated Residual Recursion}

The appendix uses the batch gate requested in the main discussion: one
Bernoulli variable is sampled for the whole batch,
\begin{equation}
Z_k\sim\operatorname{Bernoulli}(p_k).
\label{eq:app_batch_gate_var}
\end{equation}
The corrected batch is
\begin{equation}
\hat{\mathcal B}_k
=
\begin{cases}
\mathcal P_{I_k,\mathcal H_k}(\tilde{\mathcal B}_k),
& Z_k=1,\\
\tilde{\mathcal B}_k,
& Z_k=0.
\end{cases}
\label{eq:app_batch_gate}
\end{equation}

\begin{thm}[Residual recursion under the batch gate]
\label{thm:app_batch_recursion}
Under Assumptions~\ref{ass:app_rollout}--\ref{ass:app_cfs_contract},
one reverse step of MD-COAS satisfies
\begin{equation}
\mathbb E[R(\hat{\mathcal B}_k)]
\le
(1-\alpha_k p_k)\mathbb E[R(\mathcal B_k)]
+\delta_{m,k}+p_k\delta_{q,k}.
\label{eq:app_batch_recursion}
\end{equation}
If, over a phase of denoising,
$\alpha_k\ge\underline\alpha>0$,
$p_k\ge\underline p>0$,
$\delta_{m,k}\le\bar\delta_m$, and
$\delta_{q,k}\le\bar\delta_q$, then the residual sequence in that phase
obeys
\begin{equation}
\limsup_{\ell\to\infty}
\mathbb E[R_\ell]
\le
\frac{\bar\delta_m+\underline p\,\bar\delta_q}
{\underline\alpha\,\underline p}.
\label{eq:app_batch_floor}
\end{equation}
\end{thm}

\begin{proof}
Fix the history $\mathcal F_k$ and write
\[
\tilde R_k
=
\mathbb E
\!\left[
R(\tilde{\mathcal B}_k)\mid\mathcal F_k
\right].
\]
By Assumption~\ref{ass:app_mcsa},
\begin{equation}
\tilde R_k
\le
R(\mathcal B_k)+\delta_{m,k}.
\label{eq:app_mcsa_step}
\end{equation}
Conditioning on the batch gate gives two cases. If $Z_k=0$, then
$\hat{\mathcal B}_k=\tilde{\mathcal B}_k$, so the conditional expected
residual is $\tilde R_k$. If $Z_k=1$, then
Assumption~\ref{ass:app_cfs_contract} gives
\begin{equation}
\mathbb E
\!\left[
R(\hat{\mathcal B}_k)
\mid
\mathcal F_k,Z_k=1
\right]
\le
(1-\alpha_k)\tilde R_k+\delta_{q,k}.
\label{eq:app_gate_one}
\end{equation}
Since the same $Z_k$ is applied to the whole batch, the law of total
expectation over this single gate yields
\begin{equation}
\begin{aligned}
&
\mathbb E
\!\left[
R(\hat{\mathcal B}_k)
\mid\mathcal F_k
\right] \\
&=
(1-p_k)
\mathbb E
\!\left[
R(\hat{\mathcal B}_k)
\mid\mathcal F_k,Z_k=0
\right] \\
&\quad+
p_k
\mathbb E
\!\left[
R(\hat{\mathcal B}_k)
\mid\mathcal F_k,Z_k=1
\right] \\
&\le
(1-p_k)\tilde R_k
+p_k\bigl((1-\alpha_k)\tilde R_k+\delta_{q,k}\bigr)\\
&=
(1-\alpha_kp_k)\tilde R_k
+p_k\delta_{q,k}.
\end{aligned}
\label{eq:app_gate_average}
\end{equation}
Substituting \eqref{eq:app_mcsa_step} into
\eqref{eq:app_gate_average} gives
\[
\mathbb E
\!\left[
R(\hat{\mathcal B}_k)
\mid\mathcal F_k
\right]
\le
(1-\alpha_kp_k)R(\mathcal B_k)
+(1-\alpha_kp_k)\delta_{m,k}
+p_k\delta_{q,k}.
\]
Because $0\le1-\alpha_kp_k\le1$, this implies
\[
\mathbb E
\!\left[
R(\hat{\mathcal B}_k)
\mid\mathcal F_k
\right]
\le
(1-\alpha_kp_k)R(\mathcal B_k)
+\delta_{m,k}
+p_k\delta_{q,k}.
\]
Taking expectation over $\mathcal F_k$ proves
\eqref{eq:app_batch_recursion}.

For the floor bound, define the progress index $\ell$ over the phase and
let
\[
C=
\frac{\bar\delta_m}{\underline\alpha\,\underline p}
+
\frac{\bar\delta_q}{\underline\alpha}.
\]
For every step in the phase,
\[
\alpha_kp_k C
\ge
\bar\delta_m+p_k\bar\delta_q
\ge
\delta_{m,k}+p_k\delta_{q,k},
\]
because $\alpha_k\ge\underline\alpha$ and
$p_k\ge\underline p$. Subtracting $C$ from
\eqref{eq:app_batch_recursion} gives
\[
\mathbb E[R_{\ell+1}]-C
\le
(1-\alpha_\ell p_\ell)
\bigl(\mathbb E[R_\ell]-C\bigr).
\]
Since $1-\alpha_\ell p_\ell\le
1-\underline\alpha\underline p<1$, the positive part satisfies
\[
\bigl[\mathbb E[R_{\ell+1}]-C\bigr]_+
\le
(1-\underline\alpha\underline p)
\bigl[\mathbb E[R_\ell]-C\bigr]_+ .
\]
Repeated application yields
\[
\bigl[\mathbb E[R_{\ell+n}]-C\bigr]_+
\le
(1-\underline\alpha\underline p)^n
\bigl[\mathbb E[R_\ell]-C\bigr]_+ .
\]
Taking the upper limit as $n\to\infty$ proves
\eqref{eq:app_batch_floor}, because
$C=(\bar\delta_m+\underline p\,\bar\delta_q)/
(\underline\alpha\,\underline p)$.
\end{proof}

\paragraph{Interpretation.}
When $p_k=0$, the bound reduces to the MCSA proposal error, so no hard
projection contraction is claimed. When $p_k>0$, the scheduled
projection contributes the multiplicative factor
$1-\alpha_kp_k$. Increasing $I_k$ and $\mathcal H_k$ can reduce the
projection floor $\delta_{q,k}$, while increasing $p_k$ invokes the
hard correction more frequently. This is exactly the role of the
adaptive scheduler in \eqref{eq:compute_budget}--\eqref{eq:compute_dual_update}.

\subsection{Approximate Primal-Dual Interpretation}

The iALM update in the main text is
\begin{equation}
\lambda_{k-1}
=
\Pi_{\mathbb R_+}
\!\left(
(1-\eta_\lambda)\lambda_k
+\rho_k\tilde r_k
\right),
\label{eq:app_lam_update}
\end{equation}
with the hysteretic penalty update in \eqref{eq:rho_update_nonmonotone}.
Thus, $\lambda_k$ increases when the batch residual persists and decays
when the residual is inactive. The CFS schedule is regulated by the
compute dual
\begin{equation}
\nu_{k-1}
=
\Pi_{\mathbb R_+}
\!\left(
\nu_k+\eta_\nu(c_k-B)
\right).
\label{eq:app_nu_update}
\end{equation}

\begin{table*}[tb!]
\centering
\scriptsize
\setlength{\tabcolsep}{4pt}
\renewcommand{\arraystretch}{1.04}
\begin{tabular}{p{0.12\textwidth}p{0.18\textwidth}p{0.28\textwidth}p{0.34\textwidth}}
\toprule
Benchmark & Category & Parameter & Value \\
\midrule
\multirow{17}{*}{Single2D}
& \multirow{7}{*}{Rollout model}
& State & $s=[x,y]^\top$ \\
& & Control & $u=[v_x,v_y]^\top$ \\
& & Dynamics
& $s^{t+1}=\Pi_{\mathcal X}(s^t+\Delta t\,u^t)$ \\
& & Time step $\Delta t$ & $0.05$ \\
& & Horizon $H$ & $64$ \\
& & Control limit $u_{\max}$ & $1.0$ \\
& & Workspace projection $\mathcal X$ & $[-2,2]^2$ \\
\cmidrule(lr){2-4}
& \multirow{7}{*}{Obstacles}
& Map window & $[-1.5,1.0]\times[-2.0,0.5]$ \\
& & Robot radius & $0.05$ \\
& & Level $L0$ & no obstacles \\
& & Levels $L1$--$L3$ & $3,6,10$ convex obstacles \\
& & Levels $L4$--$L6$ & $18,24,28$ dense convex obstacles \\
& & Levels $L7$--$L10$ & $18,24,30,36$ union obstacles \\
& & Union obstacle size & $2$ primitives each \\
\cmidrule(lr){2-4}
& \multirow{3}{*}{Evaluation}
& Seeds & $10$ \\
& & Modes per seed & $20$ \\
& & Metrics & SSR, cost, time, density, nonconvexity \\
\midrule
\multirow{16}{*}{D3IL 7-DoF}
& \multirow{7}{*}{Rollout model}
& State & $[x_{\mathrm{ee}},y_{\mathrm{ee}},q_1,\ldots,q_7]^\top$ \\
& & Control & $\dot q\in\mathbb R^7$ \\
& & Joint update & $q^{t+1}=q^t+\Delta t\,\dot q^t$ \\
& & End-effector update
& $p_{\mathrm{ee}}^{t+1}=p_{\mathrm{ee}}^t+\Delta t\,J_{xy}(q^t)\dot q^t$ \\
& & Time step $\Delta t$ & $0.03$ \\
& & Horizon $H$ & $64$ for model-based planners \\
& & Joint-velocity limit & $\|\dot q\|_\infty\le0.8$ \\
\cmidrule(lr){2-4}
& \multirow{5}{*}{Obstacles and goal}
& Obstacle centers & six fixed centers in the workspace plane \\
& & Obstacle radii & $0.03,0.04,0.05$ \\
& & End-effector radius & $0.01$ \\
& & Target set & horizontal line at $y=0.35$ \\
& & Success condition & reaches target line and remains collision-free \\
\cmidrule(lr){2-4}
& \multirow{4}{*}{Evaluation}
& Seeds & $5$ \\
& & Modes per seed & $20$ \\
& & Pretrained-baseline horizon & $8$ with plan-once chunking \\
& & Metrics & SSR, violation, time, path length \\
\bottomrule
\end{tabular}
\caption{Experimental environment details for the two benchmarks.}
\label{tab:app_env_details}
\end{table*}

Equations \eqref{eq:app_lam_update}--\eqref{eq:app_nu_update} are
projected dual-ascent steps on the feasibility residual and the average
compute residual, respectively. Theorem~\ref{thm:app_batch_recursion}
connects these updates to the denoising residual: persistent violations
raise $(\lambda_k,\rho_k)$ and the scheduled hard-projection effort,
whereas small residuals relax enforcement and save computation. The
claim is local, because CFS gives a hard guarantee for the convexified
constraints selected in $\mathcal A_{k,m}$; global nonconvex feasibility
depends on the quality of the local half-spaces, active-set coverage,
and repeated relinearization.

\section{Experimental Environment Detail}
\label{app:env}

\subsection{2D Constrained Obstacle Avoidance}
The Single2D benchmark plans directly in the planar workspace. The
state is the planar position $s^t=[x^t,y^t]^\top$, and the action is
velocity $u^t=[v_x^t,v_y^t]^\top$. The rollout model is
\begin{equation}
s^{t+1}
=
\Pi_{\mathcal X}
\!\left(
s^t+\Delta t\,
\operatorname{clip}(u^t,-u_{\max},u_{\max})
\right),
\label{eq:app_single2d_dyn}
\end{equation}
where $\Delta t=0.05$, $T=64$, $u_{\max}=1.0$, and
$\mathcal X=[-2,2]^2$ is the box projection used by the environment.
Obstacles are generated inside the map window
$x\in[-1.5,1.0]$ and $y\in[-2.0,0.5]$. The robot radius is $0.05$.

The benchmark has levels $L0$--$L10$. Level $L0$ has no obstacles.
Levels $L1$--$L3$ contain $3$, $6$, and $10$ convex circle/box
obstacles. Levels $L4$--$L6$ contain $18$, $24$, and $28$ denser convex
obstacles, sampled from circles, boxes, and small circles. Levels
$L7$--$L10$ contain $18$, $24$, $30$, and $36$ non-convex union
obstacles, each formed by two primitives. All maps enforce clearance
around the start and target, and the non-convex levels additionally use
connectivity and non-convexity checks during generation. These levels
correspond to the Easy, Constrained, and Union families reported in the
main text.

For each level, we evaluate $10$ random seeds. For each seed and method,
the planner returns $20$ trajectory modes. The reported metrics are Safe
\& Success Rate (SSR), final planning cost, planning time, obstacle
density, and non-convexity statistics when applicable.

\subsection{D3IL 7-DoF Arm Avoidance}
The D3IL experiment uses a lifted planning state
\[
s^t=
[x_{\mathrm{ee}}^t,y_{\mathrm{ee}}^t,q_1^t,\ldots,q_7^t]^\top
\in\mathbb R^9,
\]
and the action is the joint velocity
$u^t=\dot q^t\in\mathbb R^7$. The model-based planners use the lifted
planning dynamics
\begin{equation}
\begin{aligned}
q^{t+1}
&=
q^t+\Delta t\,\dot q^t,\\
p_{\mathrm{ee}}^{t+1}
&=
p_{\mathrm{ee}}^t
+\Delta t\,J_{xy}(q^t)\dot q^t,
\end{aligned}
\label{eq:app_d3il_dyn}
\end{equation}
where $p_{\mathrm{ee}}=[x_{\mathrm{ee}},y_{\mathrm{ee}}]^\top$ and
$J_{xy}$ is the planar end-effector Jacobian. We use
$\Delta t=0.03$, $T=64$, and joint-velocity limit $0.8$ for the
model-based planners.

The obstacle layout matches the D3IL avoiding task: six cylindrical
obstacles are projected to 2D circles with centers
$(0.5,-0.1)$, $(0.425,0.08)$, $(0.575,0.08)$,
$(0.35,0.26)$, $(0.5,0.26)$, and $(0.65,0.26)$. The obstacle radius is
varied by level as $0.03$, $0.04$, and $0.05$ in the main experiment,
and the end-effector radius is $0.01$. The target is a horizontal line
at $y=0.35$ with four candidate target positions along the line. A
trajectory is counted as successful when it reaches the target line
within the success margin and remains collision-free.

For D3IL, we evaluate $5$ seeds and $20$ modes per seed. SafeDiffuser
and DPCC-C use the pretrained D3IL diffusion horizon $8$ and plan-once
chunking, while the model-based planners use horizon $64$ with online
rollouts. The metrics are SSR, violation rate, planning time, and path
length over safe and successful executions.

\section{Baselines Detail}
\label{app:baselines}

This section expands only the baselines that are actually evaluated in
the main text: MBD, EB-MBD, MDOC, SafeDiffuser, DPCC-C, and our two
ablation variants MD-COAS-A and MD-COAS-F.

\textbf{MBD.}
MBD~\cite{pan2024model} is the training-free model-based diffusion
optimizer on which our model-based comparisons are built. At reverse
step $k$, it samples $M$ denoised candidates from the same proposal used
in the main text,
\begin{equation}
\tilde\tau_{k,m}\sim
\Omega_k
\triangleq
\mathcal N
\!\left(
\frac{\tau_k}{\sqrt{\bar\alpha_k}},
\frac{1-\bar\alpha_k}{\bar\alpha_k}\mathbf I
\right),
\label{eq:app_mbd_proposal}
\end{equation}
rolls the corresponding controls out through the known dynamics, and
computes weights from the task cost:
\begin{equation}
w_{k,m}
=
\frac{
\exp(-\mathcal J(\tilde\tau_{k,m})/T_k)
}{
\sum_{\ell=1}^{M}
\exp(-\mathcal J(\tilde\tau_{k,\ell})/T_k)
}.
\label{eq:app_mbd_weights}
\end{equation}
The weighted mean is then used in the MCSA score estimate,
\begin{equation}
\bar\tau_k=\sum_{m=1}^{M}w_{k,m}\tilde\tau_{k,m},
\quad
\hat{\mathbf S}_k
=
-\frac{\tau_k}{1-\bar\alpha_k}
+\frac{\sqrt{\bar\alpha_k}}{1-\bar\alpha_k}\bar\tau_k .
\label{eq:app_mbd_score}
\end{equation}
MBD is therefore dynamics-aware, model-based, diffusion-based, and
non-data-driven. It is a necessary baseline because it isolates the
effect of adding safety enforcement. Its limitation in our experiments
is that the target density contains only the task objective, so obstacle
avoidance appears only indirectly through the task cost and is not
guaranteed.

\textbf{EB-MBD.}
EB-MBD~\cite{mishra2025eb} augments MBD with an emerging-barrier term.
Let
\begin{equation}
d_{\min}(\tau)
=
\min_t
\left(
\min_j \phi_j(q^t)-r_{\mathrm{robot}}
\right)
\label{eq:app_eb_dmin}
\end{equation}
be the minimum clearance along a rollout. EB-MBD uses a scheduled offset
$o_k$ and adds a logarithmic barrier of the form
\begin{equation}
\mathcal B_k(\tau)
=
\begin{cases}
-\mu\log\!\left(\operatorname{clip}
(d_{\min}(\tau)+o_k,\epsilon,1)\right),
& d_{\min}(\tau)+o_k>0,\\
\infty,
& \text{otherwise}.
\end{cases}
\label{eq:app_eb_barrier}
\end{equation}
The diffusion weights are then computed from
$\mathcal J(\tau)+\mathcal B_k(\tau)$, equivalently
\begin{equation}
p_k^{\mathrm{EB}}(\tau)
\propto
\exp\!\left(
-\frac{\mathcal J(\tau)+\mathcal B_k(\tau)}{T_k}
\right).
\label{eq:app_eb_target}
\end{equation}
This baseline is relevant because it represents a soft safety prior
inside model-based diffusion. Unlike MD-COAS, EB-MBD keeps the barrier
schedule fixed and does not solve a hard projection QP that can correct
a rollout back into a local feasible set.

\textbf{MDOC.}
MDOC~\cite{he2026} is the hard-filter model-based diffusion baseline. It keeps the MBD
sampling and score update, but inserts a per-step CBF-QP safety filter
inside the rollout. For a signed-distance barrier
\begin{equation}
h_j(q^t)=\phi_j(q^t)-r_{\mathrm{robot}},
\label{eq:app_mdoc_h}
\end{equation}
the single-integrator CBF constraint can be written as
\begin{equation}
\nabla h_j(q^t)^\top u^t
\ge
-\eta h_j(q^t),
\label{eq:app_mdoc_cbf}
\end{equation}
and the filtered action is obtained from
\begin{equation}
\begin{aligned}
u_{\mathrm{safe}}^t
\in
\mathop{\mathrm{arg\,min}}_{u}
&\quad
\frac{1}{2}\|u-\tilde u^t\|_2^2\\
\mathrm{s.t.}
&\quad
\nabla h_j(q^t)^\top u
\ge
-\eta h_j(q^t),
\quad \forall j\in\mathcal J_t .
\end{aligned}
\label{eq:app_mdoc_qp}
\end{equation}
In D3IL, the workspace gradient is lifted to joint velocity by
$J_{xy}(q^t)$:
\begin{equation}
\nabla h_j(p_{\mathrm{ee}}^t)^\top
J_{xy}(q^t)\dot q^t
\ge
-\eta h_j(p_{\mathrm{ee}}^t).
\label{eq:app_mdoc_joint}
\end{equation}
MDOC is a strong baseline for testing whether hard safety filtering
alone is sufficient. The key difference is that MDOC corrects each time
step independently with a CBF-QP, while MD-COAS solves a trajectory-level
CFS-QP over the full control sequence and combines this hard correction
with the adaptive iALM soft prior.

\textbf{SafeDiffuser.}
SafeDiffuser~\cite{xiao2023safediffuser} is a model-free diffusion
planner trained from data. A learned denoiser proposes the next
trajectory/action sample, and a CBF-style QP modifies the denoising
output to maintain safety. A representative correction has the form
\begin{equation}
\begin{aligned}
a_{\mathrm{safe}}^t
\in
\mathop{\mathrm{arg\,min}}_{a}
&\quad
\frac{1}{2}\|a-a_\theta^t\|_2^2\\
\mathrm{s.t.}
&\quad
h(s^{t+1})\ge(1-\kappa)h(s^t),
\quad
s^{t+1}=f(s^t,a).
\end{aligned}
\label{eq:app_safediffuser_qp}
\end{equation}
This baseline is included because it represents pretrained safe
diffusion planning with invariance-style correction. It differs from
MD-COAS in two ways: its score comes from an offline-trained diffusion
model rather than online MCSA rollouts, and its safety correction is not
the adaptive iALM+CFS mechanism used by our model-based planner.

\begin{table*}[tb!]
\small
%==============================================================================
\newcolumntype{O}{>{          \arraybackslash}m{0.35 cm}}
\newcolumntype{A}{>{          \arraybackslash}m{3.3 cm}}
\newcolumntype{B}{>{\centering\arraybackslash}m{3.7 cm}}
\newcolumntype{C}{>{\centering\arraybackslash}m{1.45 cm}}
\newcolumntype{D}{>{\centering\arraybackslash}m{1.45 cm}}
\newcolumntype{E}{>{\centering\arraybackslash}m{1.45 cm}}
\newcolumntype{F}{>{\centering\arraybackslash}m{1.6 cm}}
\newcolumntype{G}{>{\centering\arraybackslash}m{4 cm}}
\setlength\tabcolsep{0pt}
\renewcommand{\arraystretch}{1.5}
\providecommand{\cm}{\checkmark}
\providecommand{\xm}{$\times$}
\begin{center}
\begin{tabular}{O|ABCDEFG}
\toprule

\multicolumn{2}{c}{\textbf{Method}}
& {\textbf{Paradigm}}
& {Dynamics \smash{$\hphantom{^{X}}$}Aware}
& {Model- \smash{$\hphantom{^{X}}$}Based}
& {Non-Data \smash{$\hphantom{^{X}}$}Driven}
& {Smoothness \smash{$\hphantom{^{X}}$}Guarantee}
& {Constraint Handling}\\
\midrule

\parbox[t]{2mm}{\multirow{2}{*}{\rotatebox[origin=c]{90}
{\textbf{Model-Free}\hspace{0em}}}}
& ~~SafeDiffuser~\cite{xiao2023safediffuser}
& Pretrained Diffusion + CBF
& \xm & \xm & \xm & \xm
& CBF-Based Correction \\
& ~~DPCC-C~\cite{romer2024diffusion}
& Pretrained Diffusion + Projection
& $\mathcal{P}$ & \xm & \xm & \xm
& Tightening + Projection \\
\midrule

\parbox[t]{2mm}{\multirow{6}{*}{\rotatebox[origin=c]{90}
{\textbf{Model-Based Diffusion}\hspace{2em}}}}
& ~~MBD~\cite{pan2024model}
& MCSA Diffusion
& \cm & \cm & \cm & \xm
& N/A \\
& ~~EB-MBD~\cite{mishra2025eb}
& Diffusion + EB Prior
& \cm & \cm & \cm & \xm
& Soft Emerging Barrier \\
& ~~MDOC~\cite{he2026}
& Diffusion + CBF-QP
& \cm & \cm & \cm & \xm
& Per-Step Hard CBF-QP \\
\cmidrule{2-8}
& ~~\textbf{MD-COAS-A (Ours)}
& Diffusion + Adaptive iALM
& \cm & \cm & \cm & \xm
& Adaptive Soft Constraints \\
& ~~\textbf{MD-COAS-F (Ours)}
& Diffusion + Fixed CFS-QP
& \cm & \cm & \cm & \cm
& Adaptive Soft Constraints + Fixed Hard CFS-QP \\
\cmidrule{2-8}
& ~~\textbf{MD-COAS (Ours)}
& Diffusion + CO + AS
& \cm & \cm & \cm & \cm
& Adaptive iALM + Hard CFS-QP \\

\bottomrule

\end{tabular}
\caption{Comparison of the diffusion motion-planning
baselines and our ablations. The table follows the visual style of the
related-work template but only includes methods that are compared in the
main experiments. $\mathcal{P}$ denotes partially aware. }
\label{tab:app_baseline_compare}
\vspace{-0.5cm}
\end{center}
\end{table*}

\textbf{DPCC-C.}
DPCC-C~\cite{romer2024diffusion} is a diffusion predictive control
baseline for constrained planning with a pretrained diffusion prior. It
handles constraints by tightening the feasible set,
\begin{equation}
\mathcal C_\delta
=
\{\tau:\ g_i(\tau)\le-\delta,\ \forall i\},
\label{eq:app_dpcc_tight}
\end{equation}
and projecting candidate denoising outputs onto the tightened set,
\begin{equation}
\Pi_{\mathcal C_\delta}(\tau)
=
\mathop{\mathrm{arg\,min}}_{z}
\frac{1}{2}\|z-\tau\|_2^2
\quad
\mathrm{s.t.}\quad
z\in\mathcal C_\delta .
\label{eq:app_dpcc_proj}
\end{equation}
The ``C'' variant selects the candidate with minimum cumulative
projection cost,
\begin{equation}
\tau^\star
=
\mathop{\mathrm{arg\,min}}_{\tau^i}
\sum_k
\left\|
\Pi_{\mathcal C_\delta}(\tau_k^i)-\tau_k^i
\right\|_2^2 .
\label{eq:app_dpcc_select}
\end{equation}
DPCC-C is relevant because it combines diffusion sampling with explicit
constraint correction. The difference is that it is data-driven and
uses a fixed projection/tightening mechanism, while MD-COAS is
training-free and adapts both soft and hard constraint enforcement
across the reverse process.

\textbf{MD-COAS-A (Ours).}
MD-COAS-A is our soft-prior ablation. It keeps the adaptive iALM target
\begin{equation}
p_k^{A}(\tau)
\propto
\exp\!\left(
-\frac{
\mathcal J(\tau)
+\lambda_k^\top\bar g(\tau)
+\frac{\rho_k}{2}\|\bar g(\tau)\|_2^2
}{\beta}
\right),
\label{eq:app_mdcoasa}
\end{equation}
but disables the CFS-QP by setting $p_k=0$ and
$I_k=\mathcal H_k=0$. It tests whether adaptive soft feasibility alone
can solve tight non-convex maps. The main-text results show that this
helps convergence but cannot provide strict collision correction in the
hardest maps.

\textbf{MD-COAS-F (Ours).}
MD-COAS-F is our fixed-schedule hard-projection ablation. It uses the
same iALM+CFS components as MD-COAS, but fixes
\begin{equation}
\lambda_k=\lambda,\quad
\rho_k=\rho,\quad
p_k=1,\quad
I_k=I,\quad
\mathcal H_k=\mathcal H .
\label{eq:app_mdcoasf}
\end{equation}
It tests whether adding CFS projection is enough without adaptive
scheduling. It is safer than soft-only methods, but it spends projection
effort uniformly even when the batch is already feasible.

\textbf{MD-COAS (Ours).}
The full method uses both the adaptive iALM soft prior and the
trajectory-level CFS hard operator. It updates
$(\lambda_k,\rho_k)$ from the residual
\eqref{eq:dual_update_nonmonotone}--\eqref{eq:rho_update_nonmonotone}
and schedules $(p_k,I_k,\mathcal H_k)$ under the compute dual
\eqref{eq:compute_dual_update}. Thus, MD-COAS is the only evaluated
method that is simultaneously model-based, training-free,
diffusion-based, equipped with a trajectory-level hard projection, and
adaptive across the denoising process.

\section{Experiment Hyperparameters}
\label{app:hparams}

We summarize the experimental hyperparameters in four tables. Table
\ref{tab:app_common_hparams} reports the shared planning and diffusion
settings, including state/action dimensions, horizon, reverse diffusion
steps, rollout count, temperature, diffusion noise schedule, and
evaluation seeds/modes. Table~\ref{tab:app_baseline_hparams} lists the
method-specific parameters for the evaluated baselines, including MBD,
EB-MBD, MDOC, SafeDiffuser, and DPCC-C. Table
\ref{tab:app_ours_hparams} reports the hyperparameters of our ablation
variants and full MD-COAS, including the iALM initialization, CFS
projection ranges, QP tolerance, and environment-specific margins.
Finally, Table~\ref{tab:app_scheduler_hparams} gives the adaptive
scheduler parameters shared by MD-COAS across benchmarks.

\begin{table*}[tb!]
\centering
\scriptsize
\setlength{\tabcolsep}{4pt}
\renewcommand{\arraystretch}{1.2}
\begin{tabular}{p{0.19\textwidth}p{0.18\textwidth}cccccccccc}
\toprule
Benchmark
& Planner class
& State
& Act.
& $H$
& $\Delta t$
& $K$
& $M$
& $T_k$
& $\beta_0$
& $\beta_T$
& Seeds/Modes \\
\midrule
Single2D
& Model-based diffusion
& $2$
& $2$
& $64$
& $0.05$
& $100$
& $64$
& $0.5$
& $10^{-5}$
& $10^{-2}$
& $10/20$ \\
\multirow{2}{*}{D3IL 7-DoF Arm Avoidance}
& Model-based diffusion
& $9$
& $7$
& $64$
& $0.03$
& $100$
& $64$
& $0.1$
& $10^{-4}$
& $10^{-2}$
& $5/20$ \\
& Pretrained diffusion
& learned
& learned
& $8$
& $1.0$
& $20$
& --
& --
& \multicolumn{2}{c}{pretrained}
& $5/20$ \\
\bottomrule
\end{tabular}
\caption{Common planning and diffusion hyperparameters. $H$ is the
planning horizon, $K$ is the number of reverse diffusion steps, and $M$
is the number of model rollouts per reverse step.}
\label{tab:app_common_hparams}
\end{table*}

\begin{table*}[tb!]
\centering
\scriptsize
\setlength{\tabcolsep}{4pt}
\renewcommand{\arraystretch}{1.05}
\begin{tabular}{p{0.14\textwidth}p{0.36\textwidth}p{0.21\textwidth}p{0.21\textwidth}}
\toprule
Method & Parameter & Single2D & D3IL arm \\
\midrule
\multicolumn{4}{l}{\textit{Model-based diffusion baselines}}\\
\midrule
\multirow{3}{*}{MBD}
& Terminal weight $w_T$ & $100$ & -- \\
& Action noise $\sigma_u$ & $0.3$ & -- \\
& Safety layer & none & -- \\
\cmidrule(lr){1-4}
\multirow{5}{*}{EB-MBD}
& Barrier coefficient $\mu$ & $50.0$ & $5.0$ \\
& Barrier exponent $\alpha$ & $1.0$ & $2.0$ \\
& Barrier bound & $0.8$ & $0.15$ \\
& Barrier margin & $0.05$ & $0.01$ \\
& Terminal weight $w_T$ & $100$ & $170$ \\
\cmidrule(lr){1-4}
\multirow{9}{*}{MDOC}
& CBF-QP form & per-step planar & per-step joint-lift \\
& CBF activation $\tau_{\mathrm{cbf}}$ & $0.005$ & $0.01$ \\
& CBF gain $\eta_{\mathrm{cbf}}$ & $1.5$ & $1.0$ \\
& CBF margin & $0.04$ & $0.01$ \\
& Base diffusion $\beta$ & $0.05$ & $0.05$ \\
& Terminal weight $w_T$ & $100$ & $160$ \\
& Guide weight $w_g$ & -- & $20$ \\
& QP probability $p_{\mathrm{QP}}$ & $1.0$ & $1.0$ \\
& QP iterations $I_{\mathrm{QP}}$ & $10$ & $10$ \\
\midrule
\multicolumn{4}{l}{\textit{Pretrained diffusion baselines}}\\
\midrule
\multirow{5}{*}{SafeDiffuser}
& Horizon & -- & $8$ \\
& Reverse steps & -- & $20$ \\
& Batch size & -- & $8$ \\
& Safety correction & -- & all denoising steps \\
& Plan-once chunks & -- & $25$ \\
\cmidrule(lr){1-4}
\multirow{6}{*}{DPCC-C}
& Horizon & -- & $8$ \\
& Reverse steps & -- & $20$ \\
& Batch size & -- & $4$ \\
& Constraint handling & -- & tightened projection \\
& Projection solver & -- & SciPy \\
& Plan-once chunks & -- & $25$ \\
\bottomrule
\end{tabular}
\caption{Baseline hyperparameters. The dash indicates that the method is
not evaluated on that benchmark.}
\label{tab:app_baseline_hparams}
\end{table*}

\begin{table*}[tb!]
\centering
\scriptsize
\setlength{\tabcolsep}{4pt}
\renewcommand{\arraystretch}{1.05}
\begin{tabular}{p{0.14\textwidth}p{0.36\textwidth}p{0.21\textwidth}p{0.21\textwidth}}
\toprule
Method & Parameter & Single2D & D3IL 7-DoF Arm Avoidance \\
\midrule
\multirow{10}{*}{MD-COAS-A}
& Initial dual $\lambda_0$ & $0$ & -- \\
& Initial penalty $\rho_0$ & $1$ & -- \\
& Projection probability $p_{\min}$ & $0$ & -- \\
& Projection probability $p_{\max}$ & $0$ & -- \\
& Active-set size $\mathcal H_{\min}$ & $0$ & -- \\
& Active-set size $\mathcal H_{\max}$ & $0$ & -- \\
& CFS iterations $I_{\min}$ & $0$ & -- \\
& CFS iterations $I_{\max}$ & $0$ & -- \\
& QP tolerance $\epsilon_{\min},\epsilon_{\max}$ & $0,0$ & -- \\
& CFS margin & $0.05$ & -- \\
\cmidrule(lr){1-4}
\multirow{9}{*}{MD-COAS-F}
& Fixed dual $\lambda$ & $300$ & -- \\
& Fixed penalty $\rho$ & $500$ & -- \\
& Projection probability $p$ & $1.0$ & $1.0$ \\
& Active-set size $\mathcal H$ & $8$ & $8$ \\
& CFS iterations $I$ & $5$ & $5$ \\
& QP tolerance $\epsilon$ & $10^{-4}$ & $10^{-4}$ \\
& CFS margin & $0.25$ & $0.35$ \\
& Max constraints per point & $10$ & $8$ \\
& Scheduler margin & $0.05$ & $0.02$ \\
\cmidrule(lr){1-4}
\multirow{13}{*}{MD-COAS}
& Initial dual $\lambda_0$ & $0$ & $0$ \\
& Initial penalty $\rho_0$ & $1$ & $1$ \\
& Projection probability $p_{\min}$ & $0$ & $0$ \\
& Projection probability $p_{\max}$ & $1$ & $1$ \\
& Active-set size $\mathcal H_{\min}$ & $1$ & $1$ \\
& Active-set size $\mathcal H_{\max}$ & $8$ & $8$ \\
& CFS iterations $I_{\min}$ & $1$ & $1$ \\
& CFS iterations $I_{\max}$ & $5$ & $5$ \\
& QP tolerance $\epsilon_{\min}$ & $10^{-5}$ & $10^{-5}$ \\
& QP tolerance $\epsilon_{\max}$ & $10^{-2}$ & $10^{-2}$ \\
& CFS margin & $0.05$ & $0.35$ \\
& Max constraints per point & $5$ & $8$ \\
& Scheduler margin & $0.05$ & $0.02$ \\
\bottomrule
\end{tabular}
\caption{Hyperparameters for our ablation variants and full MD-COAS.
The dash indicates that the variant is not evaluated on that benchmark.}
\label{tab:app_ours_hparams}
\end{table*}

\begin{table*}[tb!]
\centering
\scriptsize
\setlength{\tabcolsep}{4pt}
\renewcommand{\arraystretch}{1.05}
\begin{tabular}{p{0.20\textwidth}p{0.36\textwidth}p{0.16\textwidth}p{0.20\textwidth}}
\toprule
Scheduler block & Parameter & Value & Main-text definition \\
\midrule
\multirow{2}{*}{Residual signal}
& Quantile level for $r_k$ & $0.9$ & Eq.~\eqref{eq:primal_residual} \\
& Dead-zone tolerance $\xi$ & $5\times10^{-4}$ & Eq.~\eqref{eq:deadzone_def} \\
\cmidrule(lr){1-4}
\multirow{7}{*}{iALM update}
& Dual forgetting $\eta_\lambda$ & $0.02$ & Eq.~\eqref{eq:dual_update_nonmonotone} \\
& Penalty lower bound $\rho_{\min}$ & $0.5$ & Eq.~\eqref{eq:rho_update_nonmonotone} \\
& Penalty upper bound $\rho_{\max}$ & $500$ & Eq.~\eqref{eq:rho_update_nonmonotone} \\
& High residual threshold $\bar r$ & $5\times10^{-4}$ & Eq.~\eqref{eq:rho_update_nonmonotone} \\
& Low residual threshold $\underline r$ & $10^{-4}$ & Eq.~\eqref{eq:rho_update_nonmonotone} \\
& Penalty increase factor $\gamma_\uparrow$ & $2.0$ & Eq.~\eqref{eq:rho_update_nonmonotone} \\
& Penalty decrease factor $\gamma_\downarrow$ & $1.5$ & Eq.~\eqref{eq:rho_update_nonmonotone} \\
\cmidrule(lr){1-4}
\multirow{5}{*}{Hard projection budget}
& Projection probability $p_k$ & scheduled in $[0,1]$ & Eq.~\eqref{eq:qp_gate} \\
& CFS iterations $I_k$ & scheduled & Eq.~\eqref{eq:qp_gate} \\
& Active-set size $\mathcal H_k$ & scheduled & Eq.~\eqref{eq:qp_gate} \\
& Expected effort budget $B$ & $8.0$ & Eq.~\eqref{eq:compute_budget} \\
& Cost proxy $c_k$ & $p_kI_k\mathcal H_k$ & Eq.~\eqref{eq:compute_budget} \\
\cmidrule(lr){1-4}
\multirow{2}{*}{Compute dual}
& Compute dual $\nu_k$ & projected to $\mathbb R_+$ & Eq.~\eqref{eq:compute_dual_update} \\
& Compute dual step $\eta_\nu$ & $0.05$ & Eq.~\eqref{eq:compute_dual_update} \\
\bottomrule
\end{tabular}
\caption{Adaptive scheduler quantities used by MD-COAS and explicitly
defined in the main text. Environment-specific CFS ranges and margins
are listed in Table~\ref{tab:app_ours_hparams}.}
\label{tab:app_scheduler_hparams}
\end{table*}

\end{document}